%% file: arxiv_v2.tex
\documentclass{article}
\usepackage{graphicx} %
\usepackage[utf8]{inputenc}
\usepackage{geometry}
\usepackage[T1]{fontenc}    %
\usepackage{xcolor}         %
\usepackage{hyperref}       %
\hypersetup{colorlinks=true,linkcolor=blue!45!black,citecolor=green!35!black,urlcolor=blue!45!black}
\usepackage{url}            %
\usepackage{booktabs}       %
\usepackage{amsfonts}       %
\usepackage{nicefrac}       %
\usepackage{microtype}      %
\usepackage{amsmath}
\usepackage{amssymb}
\usepackage{mathtools}
\usepackage{amsthm}

\usepackage{cleveref}
\Crefname{algorithm}{Algorithm}{Algorithms}
\Crefname{assumption}{Assumption}{Assumptions}
\Crefname{equation}{}{}
\Crefname{figure}{Fig.}{Figs.}
\Crefname{table}{Table}{Tables}
\Crefname{section}{Section}{Sections}
\Crefname{subsection}{Section}{Sections}
\Crefname{theorem}{Theorem}{Theorems}
\Crefname{lemma}{Lemma}{Lemmmas}
\Crefname{proposition}{Proposition}{Propositions}
\Crefname{definition}{Definition}{Definitions}
\Crefname{corollary}{Corollary}{Corollaries}
\Crefname{remark}{Remark}{Remarks}
\Crefname{example}{Example}{Examples}
\Crefname{appendix}{Appendix}{Appendices}

\input{notation}

\title{Optimal Rates for Pure $\varepsilon$-Differentially Private Stochastic Convex Optimization with Heavy Tails}
\author{
Andrew Lowy\\
CISPA Helmholtz Center for Information Security\\
\texttt{lowy@cispa.de}
}
\date{}

\begin{document}

\maketitle

\begin{abstract}
We study stochastic convex optimization (SCO) with heavy-tailed gradients under pure \(\eps\)-differential privacy (DP). Instead of assuming a bound on the worst-case Lipschitz parameter of the loss, we assume only a bounded \(k\)-th moment. This assumption allows for unbounded, heavy-tailed stochastic gradient distributions, and can yield sharper excess risk bounds. 
Prior work characterized the minimax optimal rate for $\rho$-zero-concentrated DP SCO up to logarithmic factors in this setting, but the pure $\eps$-DP case has remained open. 
We characterize the minimax optimal excess-risk rate for pure \(\eps\)-DP heavy-tailed SCO up to logarithmic factors. Our algorithm achieves this rate in polynomial time with high probability. Moreover, it runs in deterministic polynomial time when the worst-case Lipschitz parameter is polynomially bounded. For important structured problem classes --- including hinge/ReLU-type and absolute-value losses on Euclidean balls, ellipsoids, and polytopes --- we achieve deterministic polynomial time even when the worst-case Lipschitz parameter is infinite. Our approach is based on a novel framework for privately optimizing \textit{Lipschitz extensions} of the empirical loss. We complement our upper bound with a nearly matching high-probability lower bound.
\end{abstract}

\section{Introduction}

\textit{Stochastic convex optimization} (SCO) is a fundamental problem in machine learning.
Given i.i.d.\ data \(Z=(z_1,\dots,z_n)\) drawn from an unknown distribution \(P\), the goal is to solve
\begin{equation}
\label{eq:SCO}
\min_{w\in\WW} \left\{F(w) := \E_{z\sim P}[f(w,z)] \right\},
\end{equation}
where \(\WW\subset\mathbb{R}^d\) is a convex compact parameter domain of $\ell_2$-diameter $D$ and
\(f(\cdot,z)\) is a convex loss function. The quality of a solution $w$ of~\eqref{eq:SCO} is measured by its \textit{excess risk} $F(w) - \min_{w'\in\WW} F(w')$.

In many applications, the data used to train models contains sensitive information. \textit{Differential privacy} (DP) provides a rigorous framework for protecting individual data contributions while enabling statistical learning \cite{dwork2006calibrating}. 
The strongest standard notion of DP is pure $\eps$-DP.

A large body of work has studied DP SCO under the \textit{uniform Lipschitz} assumption\footnote{Throughout, \(\nabla f(w,z)\) denotes the gradient with respect to \(w\) when \(f(\cdot,z)\) is differentiable at \(w\), and otherwise denotes an arbitrary subgradient in \(\partial f(\cdot,z)(w)\).}:
\begin{equation}
\label{eq:uniformLip}
    \sup_{w\in\WW,z\in\ZZ}\|\nabla f(w,z)\|\le L.
\end{equation}
Under this assumption, optimal excess risk bounds scale with the worst-case Lipschitz parameter \(L\) \cite{bft19,fkt20,AsiFeKoTa21,bassily2023differentially,lowy2024faster}. In many applications, however, \(L\) can be extremely large or infinite, making such guarantees overly pessimistic or vacuous. For example, in linear regression with squared loss, the gradient norm scales with the squared feature norm, so if the feature distribution is unbounded then \(L\) is unbounded as well.

To address these issues, a recent line of work studies \textit{heavy-tailed DP SCO} under weaker moment assumptions on
the gradients~\cite{wx20,asi2021private,tao2021optimal,kamath2022improved,hu2022high,wang2022differentially,asi2024private,lowy2025private}.
Instead of requiring uniformly bounded gradients, one assumes a \textit{bounded \(k\)-th moment}:
\begin{equation}
\label{eq:momentbound}
\E_{z\sim P}\!\left[\sup_{w\in\WW}\|\nabla f(w,z)\|^k\right] \le G_k^k
\end{equation}
for some \(k\ge2\). For nonsmooth $f$, the supremum in~\eqref{eq:momentbound} is taken also over subgradient norms. \eqref{eq:momentbound} allows unbounded heavy-tailed gradient distributions while controlling the average behavior, and captures a broad class of realistic learning problems. Note $G_2 \le G_k \le L$ for all $k$ and often $G_k \ll L$: e.g., for linear regression, $G_k$ scales with the $2k$-th moment of the feature data. 
Thus, excess risk bounds that scale with $G_k$ instead of $L$ are often sharper. For $\rho$-zCDP~\cite{bun2016concentrated} heavy-tailed SCO, the optimal excess risk rate is known up to logarithmic factors~\cite{asi2024private,lowy2025private}.

\paragraph{The pure-DP gap.}
Despite this progress, the case of \emph{pure} \(\eps\)-differential privacy remains poorly understood. The work of~\cite{bd14} proved an in-expectation pure DP lower bound, but no direct algorithm achieving this rate was previously known. Existing heavy-tailed DP SCO algorithms rely on noisy clipped-gradient methods, which are suboptimal under pure \(\eps\)-DP. This raises the following:

\begin{center}
\fbox{
\parbox{0.72\linewidth}{
\textbf{Question 1.}
What is the minimax optimal excess risk for heavy-tailed stochastic convex optimization under pure \(\eps\)-differential privacy?
}}
\end{center}

\paragraph{Contribution 1: Optimal excess risk for pure-DP heavy-tailed SCO (up to logarithms).}
We determine the minimax optimal pure $\eps$-DP excess risk under~\eqref{eq:momentbound}, obtaining with probability $\geq 1 - \delta$ 
\begin{equation}
\label{eq:excriskupperbound}
O\left(
G_k D\!\left(\frac{d \log(1/\delta)}{n\eps}\right)^{1-\frac1k}
+
\frac{G_2D \sqrt{\log(1/\delta)}}{\sqrt n}
\right).
\end{equation}
This bound is tight up to a $O(\log(1/\delta)^{1-1/k})$ factor: we prove a nearly matching high-probability lower bound that is sharper than the in-expectation lower bound of \cite{bd14}. 

\begin{center}
\fbox{
\parbox{0.72\linewidth}{
\textbf{Question 2.}
Can the minimax optimal excess risk for pure-DP heavy-tailed SCO be achieved by a computationally efficient algorithm?
}}
\end{center}

\paragraph{Contribution 2: Polynomial-time algorithms for pure-DP heavy-tailed SCO.}
Our main result is that \eqref{eq:excriskupperbound} can be achieved in polynomial time with high probability; if the worst-case Lipschitz parameter is finite and polynomially bounded, the runtime is polynomial with probability \(1\). 
This is the first polynomial-time \(\eps\)-DP algorithm achieving the high-probability excess-risk guarantee \eqref{eq:excriskupperbound}. A black-box purification route from zCDP is looser and more cumbersome than our direct pure-DP approach; see Appendix~\ref{app:related-work}.

For certain structured subclasses --- including hinge/ReLU-type and absolute-value losses on Euclidean balls, ellipsoids, and polytopes --- we prove a stronger guarantee: deterministic polynomial time, even when the worst-case Lipschitz parameter is infinite. 

\paragraph{Contribution 3: A new framework for privately optimizing Lipschitz extensions.}
To obtain our upper bounds, we move away from clipped-gradient methods and instead use the \textit{Lipschitz extension}
\begin{equation}
\label{eq:def-LipExt}
f_C(w,z):=\inf_{y\in\WW}\bigl[f(y,z)+C\|w-y\|\bigr].
\end{equation}
This reduces heavy-tailed regularized ERM to Lipschitz regularized ERM. To optimize the resulting objective efficiently under pure DP, we develop a novel jointly convex reformulation together with adaptive (inexact) projected subgradient methods with deterministic accuracy guarantees. 

\begin{table}[t]
\centering
\caption{Heavy-tailed DP SCO: summary of results. Upper and lower bounds match up to logarithms.}
\label{tab:comparison}
\footnotesize
\begin{tabular}{p{0.50\linewidth}ccc}
\toprule
Result & Privacy & Runtime & Setting \\
\midrule
zCDP upper/lower bounds~\cite{asi2024private,lowy2025private}
& $\rho$-zCDP & polytime w.p.~1 & general \\

Our upper bound (Thm.~\ref{thm:main-dop})
& pure $\eps$-DP & polytime w.h.p. & general \\

Our structured-subclass upper bound (\Cref{cor:polyhedral-practical})
& pure $\eps$-DP & polytime w.p.~1 & structured \\

Our lower bound (\Cref{thm:sco-lb-mainbody})
& pure $\eps$-DP & --- & general \\
\bottomrule
\end{tabular}
\end{table}

\subsection{Challenges and Techniques}
\label{sec:challenges}

We build on the population-level localization framework of~\cite{asi2024private}, reducing heavy-tailed SCO to regularized ERM. The key question is how to solve heavy-tailed ERM efficiently under pure $\eps$-DP.

\paragraph{Challenge 1: Clipped gradients are insufficient.}
Prior heavy-tailed DP SCO algorithms rely on noisy clipped gradient methods. 
These are optimal for zCDP, but suboptimal under pure DP because the more forgiving zCDP composition is unavailable under pure DP. 
Our algorithm abandons gradient clipping and instead uses the \textit{Lipschitz extension}~\eqref{eq:def-LipExt}.

\paragraph{Challenge 2: Optimizing the Lipschitz extension under pure DP.}
The Lipschitz extension~\eqref{eq:def-LipExt} is defined by an inner optimization problem. 
Exact evaluation of $f_C(w,z)$ from fixed finite local oracle information is impossible
(see \Cref{thm:interior_exact_extension_impossible}), and certified approximation is nontrivial because no Lipschitz bound for the inner problem is available.
For pure DP this is especially problematic, since the required sensitivity control cannot fail even with small probability.
We address this via a \textit{jointly convex reformulation} and \textit{adaptive inexact projected subgradient methods}, which provide certified approximate minimizers without requiring prior knowledge of the Lipschitz parameter. 

\paragraph{Challenge 3: The bias of the Lipschitz extension is too large on the original domain.}

Over the full domain \(\WW\), the Lipschitz-extension bias is too large. We therefore use \textit{localized Lipschitz extensions} over the current set \(U\), so the bias scales with \(\diam(U)\), and \textit{recursively shrink \(U\)} around the regularized ERM minimizer via output perturbation. Geometric allocation of privacy and failure probabilities avoids extra log-log factors. Combining these techniques yields our main algorithm.

\paragraph{Lower bound techniques.} We construct two different hard instances: one for the private error term and one for the non-private error. To prove the private term, we combine the packing technique of \cite{bd14} with a reduction from quantile estimation to decoding. The non-private term is proved via a bounded two-point construction together with the high-probability testing framework of \cite{ma2024high}.

\subsection{Preliminaries}
Let \(\|\cdot\|\) denote the \(\ell_2\) norm. $\mathbb{B}(w_0,r)$ denotes $\ell_2$-ball of radius $r$ around $w_0$. For a function \(h:\WW\to\R\), a vector \(g\in\R^d\) is a \emph{subgradient} of \(h\) at \(w\in\WW\) if
$
h(u)\ge h(w)+\langle g,u-w\rangle~\forall u\in\WW.
$
We write \(\partial h(w)\) for the set of all subgradients of \(h\) at \(w\). When \(h\) is differentiable at \(w\), \(\partial h(w)=\{\nabla h(w)\}\), the gradient of $h$ at $w$. For \(\lambda\ge 0\), we say that \(h\) is \emph{\(\lambda\)-strongly convex} if for every \(w,u\in\WW\) and every \(g\in\partial h(w)\),
$
h(u)\ge h(w)+\langle g,u-w\rangle+\frac{\lambda}{2}\|u-w\|^2.
$
If \(\lambda=0\), we say $h$ is \textit{convex}. For a closed convex set \(K\subseteq\R^d\), the Euclidean \textit{projection} of \(y\in\R^d\) onto \(K\) is
$
\Pi_K(y):=\argmin_{u\in K}\|u-y\|.
$
Denote
$
h^*:=\min_{w\in\WW} h(w).
$
Throughout, \(c_{(\cdot)}\) denotes an absolute constant. We use \(O(\cdot)\) and \(\lesssim\) to hide absolute constants, and \(\Tilde{O}(\cdot)\) to additionally hide logarithmic factors.

Throughout the paper, we assume the following:
\begin{assumption}
\label{ass:main}
\begin{enumerate}
    \item The loss function \(f:\WW\times\ZZ\to\R\) is such that \(f(\cdot,z)\) is convex for every \(z\in\ZZ\), and \eqref{eq:momentbound} holds for some constant \(k\ge 2\) where $G_k$ is publicly known (c.f. \Cref{rem:public}).
    \item The domain \(\WW\subset\R^d\) is closed and convex with \(\ell_2\)-diameter \(D\), and is projection-friendly: for every \(y\in\R^d\), the Euclidean projection \(\Pi_{\WW}(y)\) can be computed in polynomial time.
    \item Given \(z\in\ZZ\), \(w\in\WW\), one can compute \(f(w,z)\) and \(g\in \partial f(w,z)\) in polynomial time.
\end{enumerate}
\end{assumption}
By monotonicity of moments, \(G_2\le G_k\le \sup_{w,z}\|\nabla f(w,z)\|\) for all \(k\ge2\).

\paragraph{Differential Privacy.}
Differential privacy ensures that no attacker can infer much more about any individual's data than they could have inferred had that person's data not been used. 
\begin{definition}[Pure Differential Privacy~\cite{dwork2006calibrating}]
A randomized algorithm \(\alg:\ZZ^n\to\calO\) is \emph{$\eps$-differentially private} (DP) if for every pair of neighboring datasets \(Z,Z'\in\ZZ^n\) differing in one entry, and every measurable set \(S\subseteq\calO\), $
\Pr(\alg(Z)\in S)\le e^\eps \Pr(\alg(Z')\in S).$
\end{definition}

\emph{Zero-concentrated differential privacy} (zCDP) is weaker than $\eps$-DP; see Appendix~\ref{app:zcdp}.

\section{Algorithmic Building Blocks}
\label{sec:building-blocks}
This section develops the key ingredients that will be used by our main algorithm.

\subsection{Reduction from SCO to ERM}
\label{sec:scotoERMreduction}
We use the population-localization framework of~\cite{asi2024private}, which reduces heavy-tailed SCO to a sequence of private regularized ERM problems; see Algorithm~\ref{alg:pop-localize}.

\begin{algorithm}[t]
\caption{\textsc{Pop--Localize}$(Z,\eps,\alg_{\mathrm{ERM}},\delta)$~\cite{asi2024private}}
\label{alg:pop-localize}
\DontPrintSemicolon
\KwIn{Data \(Z\in\ZZ^n\), privacy \(\eps\), private regularized ERM solver \(\alg_{\mathrm{ERM}}\), error prob. $\delta$.}
\KwOut{\(\hat w\in\WW\).}

Choose \(J=\Theta(\log(\log n/\delta))\)\;
Set number of phases $T=\left\lfloor \log_2(n/(2J))\right\rfloor $ \;
Initialize \(\bar w_1\in\WW\) arbitrarily\;
Choose base regularization \(\lambda_1>0\)\;

Partition \(Z\) into disjoint \(Z^1,\dots,Z^T\) with
$
|Z^t| = n_t := \lfloor n/2^t \rfloor,$ discarding unused samples \;

\For{\(t=1\) \KwTo \(T\)}{

    Set $
    \lambda_t := 32^{t-1}\lambda_1 .$

    Partition \(Z^t\) into \(J\) disjoint blocks
$
    Z^{(t,1)},\dots,Z^{(t,J)}
$ of size $|Z^{(t,j)}| = m_t := \lfloor n_t/J\rfloor$.

    \For{\(j=1\) \KwTo \(J\)}{
        Compute
     $  
        \hat w_{t,j}
        \gets
        \alg_{\mathrm{ERM}}(Z^{(t,j)},\eps,\lambda_t,\bar w_t).
$ 
    }

    Aggregate \(\{\hat w_{t,j}\}_{j=1}^J\) via geometric aggregation to obtain \(\bar w_{t+1}\)\;
}
\Return \(\bar w_{T+1}\)\;
\end{algorithm}

\paragraph{Guarantee for Algorithm~\ref{alg:pop-localize}.}
For a sample \(Z=(z_1,\dots,z_m)\in\ZZ^m\), a center \(w_0\in\WW\), and a
regularization parameter \(\lambda>0\), define the regularized empirical
objective
\[
\hat F_{\lambda,Z}^{(w_0)}(w)
:=
\frac1m\sum_{i=1}^m f(w,z_i)
+
\frac{\lambda}{2}\|w-w_0\|^2,
\qquad
\hat w_\lambda(Z;w_0)
:=
\argmin_{w\in\WW}\hat F_{\lambda,Z}^{(w_0)}(w).
\]

By the following theorem, it suffices to design an \(\eps\)-DP regularized ERM solver \(\alg_{\mathrm{ERM}}\) such that, on every instance \((Z,\eps,\lambda,w_0)\) with \(|Z|=m\), with probability at least \(0.6\) its output satisfies

\begin{equation}
\label{eq:distance-to-ERM}
\bigl\|
\alg_{\mathrm{ERM}}(Z,\eps,\lambda,w_0)-\hat w_\lambda(Z;w_0)
\bigr\|
\le
\frac{c_{\mathrm{erm}}}{\lambda}
\left(
G_k\Bigl(\frac{d}{m\eps}\Bigr)^{1-\frac1k}
+
\frac{G_2}{\sqrt m}
\right).
\end{equation}

\begin{theorem}[Regularized ERM implies SCO~\cite{asi2024private}]
\label{thm:localization-generic}
Fix \(\delta\in(0,1/2)\). Suppose \(\alg_{\mathrm{ERM}}\) is an \(\eps\)-DP
algorithm such that for every center \(w_0\in\WW\) and every \(\lambda>0\), its
output satisfies \eqref{eq:distance-to-ERM} with probability at least \(0.6\)
over \(\alg_{\mathrm{ERM}}\) and \(Z\sim P^m\). Then, \Cref{alg:pop-localize} is \(\eps\)-DP and there exists a choice of parameters such that, with probability at least \(1-\delta\),
\[
F(\hat w)-F^*
\lesssim 
G_kD
\left(
\frac{d\log(1/\delta)}{n\eps}
\right)^{1-\frac1k}
+
G_2D
\sqrt{\frac{\log(1/\delta)}{n}}.
\]
Moreover, if one call to \(\alg_{\mathrm{ERM}}\) on a dataset of size \(m\)
takes time \(\mathrm{Time}(\alg_{\mathrm{ERM}},m,d,\eps,\lambda)\), then the
total runtime of \Cref{alg:pop-localize} is bounded by
$\Tilde{O}\left( \mathrm{Time}(\alg_{\mathrm{ERM}},n,d,\eps,\lambda_1) \right)$.
\end{theorem}

Therefore, the rest of the paper focuses on designing $\eps$-DP regularized ERM solvers satisfying \eqref{eq:distance-to-ERM}.

\subsection{Lipschitz Extension}
\label{sec:lipext-building-block}

The Lipschitz extension \eqref{eq:def-LipExt} transforms any convex
\(f(\cdot,z)\) into a convex \(C\)-Lipschitz function.

\begin{lemma}[\cite{hiriart2013convex}]
\label{lem:lip-ext-basic}
Let \(f(\cdot,z)\) be convex on \(\WW\). Then:
\begin{enumerate}
    \item \(f_C(\cdot,z)\) is convex on \(\WW\);
    \item \(f_C(\cdot,z)\) is \(C\)-Lipschitz on \(\WW\);
    \item \(f_C(w,z)\le f(w,z)\) for all \(w\in\WW\).
\end{enumerate}
\end{lemma}

This suggests reducing heavy-tailed regularized ERM to regularized Lipschitz
ERM, provided we can control the error introduced by replacing \(f\) with
\(f_C\). For $Z = (z_1, \ldots, z_m)$, 
\(w_0\in\WW\), \(\lambda>0\), define the \textit{regularized empirical
$C$-Lipschitz-extension} 
$
\hat F^{(w_0)}_{C,\lambda,Z}(w)
:=
\frac1m\sum_{i=1}^m f_C(w,z_i)
+\frac{\lambda}{2}\|w-w_0\|^2.
$

\paragraph{Why the global extension is insufficient.}
Output perturbation gives the optimal distance bound for $C$-Lipschitz, $\lambda$-strongly convex $\eps$-DP ERM. Running it on $\hat F^{(w_0)}_{C,\lambda,Z}$ yields $w_{\mathrm{DP}}$ such that w.p. $\ge 0.6$
\[
\|w_{\mathrm{DP}} -\hat w_\lambda(Z;w_0)\|
\lesssim
\frac{Cd}{\lambda m\eps}
+
\sqrt{\frac{D G_k^k}{\lambda C^{k-1}}}.
\]
The first term is the DP optimization error and the second is the bias from replacing \(f\) by \(f_C\); see Appendix~\ref{app:lipext-building}.
Even after optimizing over \(C\), this is too large to imply
\eqref{eq:distance-to-ERM} unless the diameter \(D\) is small. Our solution is to replace the global diameter \(D\) by the diameter of a localized
search set.

\subsection{Shrinking the Lipschitz Extension Bias by Localization}
\label{sec:bias-shrinking}

To reduce the Lipschitz extension bias, we
use Lipschitz extensions localized to the current search set. For any closed convex set \(U\subseteq \WW\), define the \textit{\(U\)-localized
\(C\)-Lipschitz extension} by
\begin{equation}
\label{eq:def-localized-lipext}
f_{C,U}(w,z)
:=
\inf_{y\in U}\bigl[f(y,z)+C\|w-y\|\bigr],
\qquad w\in U.
\end{equation}
When \(U=\WW\), this is the global Lipschitz extension
\eqref{eq:def-LipExt}. As in \Cref{lem:lip-ext-basic},
\(f_{C,U}(\cdot,z)\) is convex and \(C\)-Lipschitz on \(U\), and
\(f_{C,U}(w,z)\le f(w,z)\) for all \(w\in U\). For \(Z=(z_1,\dots,z_m)\), define 
\[
\hat F^{(w_0)}_{C,U,\lambda,Z}(w)
:=
\frac1m\sum_{i=1}^m f_{C,U}(w,z_i)
+
\frac{\lambda}{2}\|w-w_0\|^2,
\qquad \hat w_{C,U,\lambda}(Z;w_0)
:=
\argmin_{w\in U}
\hat F^{(w_0)}_{C,U,\lambda,Z}(w),
\]

for $w \in U$. The key is that the bias of the localized extension is proportional to
\(\diam(U)\), not \(D\). 

\begin{lemma}[Uniform bias bound for localized Lipschitz extensions]
\label{lem:localized-lipext-bias}
With probability at least \(1-\beta\), 
\[
\|\hat w_{C,U,\lambda}(Z;w_0)-\hat w_\lambda(Z;w_0)\|
\le
\sqrt{
\frac{2R G_k^k}{\lambda\beta(k-1)C^{k-1}}
}
\]
simultaneously for every closed convex \(U\subseteq\WW\) with
\(\diam(U)\le R\) and \(\hat w_\lambda(Z;w_0)\in U\).
\end{lemma}

The uniformity over \(U\) is important: in \Cref{alg:recursive-localized-outputpert-main}, the
sets \(U\) are chosen adaptively from previous private outputs, but 
\Cref{lem:localized-lipext-bias} still applies simultaneously to all candidate
sets of diameter \(\le R\).

\paragraph{Localizing the domain via output perturbation.}
Given a current set \(U\) with diameter bound \(R\) and failure probability $\beta$, \textsc{Localized-OutputPert} (Algorithm~\ref{alg:localized-outputpert}) privately minimizes $\hat F^{(w_0)}_{C,U,\lambda,Z}(w)$
over \(U\), adds isotropic Laplace noise with
scale \(Cd/(\lambda m\eps)\), and returns a ball around the noisy point with
radius $R^+ \lesssim
\sqrt{\frac{R G_k^k}{\beta \lambda C^{k-1}}}
+
\frac{Cd \log(1/\beta)}{\lambda m\eps}.
$

\begin{algorithm}[t]
\caption{\textsc{Localized-OutputPert}$(Z,\eps,\lambda,w_0,U,R,C,\beta)$}
\label{alg:localized-outputpert}
\DontPrintSemicolon
\KwIn{Data \(Z=(z_1,\dots,z_m)\), privacy \(\eps\), regularization \(\lambda\), center \(w_0\), convex set \(U\subseteq\WW\), deterministic diameter bound \(R\ge\diam(U)\), Lipschitz parameter \(C\), failure probability \(\beta\)}
\KwOut{A center \(w_{\mathrm{loc}}\in\WW\) and localized domain \(U^+\subseteq\WW\)}

Set \(\ell=\log(1/\beta)\) and $
\alpha
=
\frac{C^2}{100\lambda m^2}
\min\left\{1,\frac{d^2\ell^2}{\eps^2}\right\}.
$

Compute any point \(\tilde w\in U\) satisfying
$
\hat F^{(w_0)}_{C,U,\lambda,Z}(\tilde w)
-
\min_{w\in U}\hat F^{(w_0)}_{C,U,\lambda,Z}(w)
\le \alpha .
$

Sample isotropic Laplace noise \(b\) with density proportional to
$
\exp\!\left(-\frac{\eps\lambda m}{10C}\|b\|_2\right).
$

Set \(w_{\mathrm{loc}}:=\Pi_{\WW}(\tilde w+b)\) and 
$
R^+
:=
100
\left[
\sqrt{\frac{R G_k^k}{\lambda\beta C^{k-1}}}
+
\frac{Cd\ell}{\lambda m\eps}
\right].
$

Set \(U^+:=\WW\cap\BB(w_{\mathrm{loc}},R^+)\)\;
\Return \((w_{\mathrm{loc}},U^+)\).
\end{algorithm}

\Cref{lem:localized-outputpert} shows that \Cref{alg:localized-outputpert} is $\eps$-DP and if
\(\hat w_\lambda(Z;w_0)\in U\), then with probability $\ge 1 - 2\beta$, $U^+$ also contains \(\hat w_\lambda(Z;w_0)\).
Choosing $C$ optimally in line 4 of \Cref{alg:localized-outputpert} yields
\[
R^+ \lesssim
R_\star(R/R_\star)^{1/(k+1)}, \qquad
R_\star:=
(G_k/\lambda)(d/(m\eps))^{1-1/k}.
\]
Thus one localization step contracts the radius toward \(R_\star\) from $R$ to
\(R^+\), but not all the way to
\(R_\star\). To achieve~\eqref{eq:distance-to-ERM}, it suffices to privately
obtain the center of a ball of radius \(O(R_\star)\) that contains
\(\hat w_\lambda(Z;w_0)\) with constant probability. We accomplish this by applying \Cref{alg:localized-outputpert} recursively.

\paragraph{Recursive localization.}
\Cref{alg:recursive-localized-outputpert-main} applies
\textsc{Localized-OutputPert} recursively. Writing
\(R_j\) for a deterministic upper bound on \(\diam(U_j)\), step \(j\)
constructs a private ball \(U_{j+1}\) around \(w_{\mathrm{loc},j+1}\) that
still contains \(\hat w_\lambda(Z;w_0)\) with high probability, while updating
\[
R_{j+1}\lesssim R_\star\left(\frac{R_j}{R_\star}\right)^{1/(k+1)}.
\]
Equivalently, if \(R_j=R_\star M_j\), then \(M_{j+1}\lesssim M_j^{1/(k+1)}\).
Hence after
\(S\approx \log_{k+1}\log(eD/R_\star)\) rounds the radius is \(O(R_\star)\).
The geometric choice of \(\eps_j,\beta_j\) avoids extra log-log factors.

\begin{algorithm}[t]
\caption{\textsc{Recursive-Localized-OutputPert}}
\label{alg:recursive-localized-outputpert-main}
\DontPrintSemicolon
\KwIn{Dataset \(Z=(z_1,\dots,z_m)\), privacy \(\eps\), regularization \(\lambda\), center \(w_0\)}
\KwOut{A private center \(w_{\mathrm{loc},S}\in\WW\)}

Set \(R_\star:=\frac{G_k}{\lambda}\left(\frac{d}{m\eps}\right)^{1-\frac1k}\)\;
\If{\(D\le R_\star\)}{
    Choose any fixed \(w_{\mathrm{loc},S}\in\WW\) and \Return \(w_{\mathrm{loc},S}\)\;
}

Set \(a:=1/(k+1)\) and
\(S:=\left\lceil \log_{k+1}\log\left(\frac{eD}{R_\star}\right)\right\rceil\)\;

Initialize \(U_0:=\WW\), \(R_0:=D\)\;

\For{\(j=0\) \KwTo \(S-1\)}{
    Set
    \(\theta_j:=\frac{(1-a)a^{S-1-j}}{1-a^S}\),
    \(\eps_j:=\frac{\eps}{2}\theta_j\),
    \(\beta_j:=\beta_0\theta_j\), where \(\beta_0=1/100\), and
    \(\ell_j:=\log(1/\beta_j)\)\;

    Set
    \(C_j:=G_k^{k/(k+1)}(\lambda R_j)^{1/(k+1)}
    \left(\frac{m\eps_j}{d\ell_j}\right)^{2/(k+1)}
    \beta_j^{-1/(k+1)}\)\;

    Run \((w_{\mathrm{loc},j+1},U_{j+1})\gets
    \textsc{Localized-OutputPert}
    (Z,\eps_j,\lambda,w_0,U_j,R_j,C_j,\beta_j)\)\;

    Set
    \(R_{j+1}:=\min\left\{D,\;
    1000
    \left(\frac{G_k}{\lambda}\right)^{k/(k+1)}
    R_j^{1/(k+1)}
    \left(\frac{d\ell_j}{m\eps_j}\right)^{(k-1)/(k+1)}
    \beta_j^{-1/(k+1)}
    \right\}\)\;
}
\Return \(w_{\mathrm{loc},S}\)\;
\end{algorithm}

\begin{proposition}
\label{prop:bias-reduced-decomp}
Let \(Z\sim P^m\), fix \(w_0\in\WW\), and \(\lambda>0\). 
\Cref{alg:recursive-localized-outputpert-main} is \(\eps\)-DP and
\[
\Pr\!\left(
\|w_{\mathrm{loc},S}-\hat w_\lambda(Z;w_0)\|
\le
c_2\frac{G_k}{\lambda}
\left(\frac{d}{m\eps}\right)^{1-\frac1k}
\right)\ge 0.7.
\]
\end{proposition}

\begin{remark}[\textbf{Summary of the reduction}]
\label{rem:summary-of-reduction}
\Cref{prop:bias-reduced-decomp} gives the distance guarantee
\eqref{eq:distance-to-ERM} required by \Cref{thm:localization-generic}, yielding the optimal statistical rate. Thus, to prove our main result, it remains to implement line 2 of \Cref{alg:localized-outputpert} \textit{efficiently}: compute an approximate minimizer of \(\hat F^{(w_0)}_{C,U,\lambda,Z}\) over \(U\).
\end{remark}

\subsection{Efficient Optimization of the Regularized Lipschitz Extension}
\label{sec:joint-adaptive}

We now give a primitive for optimizing the empirical regularized Lipschitz-extension.

\paragraph{Joint convex reformulation.}
Fix a dataset \(Z=(z_1,\dots,z_m)\), a center \(w_0\in\WW\), a regularization
parameter \(\lambda>0\), a Lipschitz parameter \(C>0\), and a closed convex set
\(U\subseteq\WW\). Define
\begin{equation}
\label{eq:localized-joint-program}
\Phi_{Z,U}(w,y_1,\dots,y_m)
:=
\frac1m\sum_{i=1}^m
\bigl[f(y_i,z_i)+C\|w-y_i\|\bigr]
+
\frac{\lambda}{2}\|w-w_0\|^2,
\qquad
(w,y_1,\dots,y_m)\in U^{m+1}.
\end{equation}

This reformulation avoids evaluating each \(f_{C,U}(w,z_i)\): all inner minimizations defining the Lipschitz extensions are folded into one convex program over \((w,y_1,\ldots,y_m)\). \Cref{lem:localized-joint-convex} reduces minimization of the localized
Lipschitz-extension objective \(\hat F^{(w_0)}_{C,U,\lambda,Z}\) over \(U\) to
minimization of \(\Phi_{Z,U}\) over \(U^{m+1}\).

\begin{lemma}[Localized joint convex reformulation]
\label{lem:localized-joint-convex}
The function \(\Phi_{Z,U}\) is convex on \(U^{m+1}\), and
\[
\min_{(w,y_1,\dots,y_m)\in U^{m+1}}\Phi_{Z,U}(w,y_1,\dots,y_m)
=
\min_{w\in U}\hat F^{(w_0)}_{C,U,\lambda,Z}(w).
\]
Moreover, if
$
u_\alpha=(w_\alpha,y_{1,\alpha},\dots,y_{m,\alpha})\in U^{m+1}
$
satisfies
\[
\Phi_{Z,U}(u_\alpha)-\min_{u\in U^{m+1}}\Phi_{Z,U}(u)\le \alpha,
\]
then
\[
\hat F^{(w_0)}_{C,U,\lambda,Z}(w_\alpha)
-
\min_{w\in U}\hat F^{(w_0)}_{C,U,\lambda,Z}(w)
\le \alpha.
\]
\end{lemma}

\paragraph{Certified adaptive projected subgradient method solver.}
We do not know the Lipschitz constant of \(\Phi\), since \(f(\cdot,z)\) may have unbounded worst-case Lipschitz parameter. Standard first-order methods therefore do not directly yield a certified \(\alpha\)-approximate minimizer. We instead use the adaptive projected subgradient method in Algorithm~\ref{alg:adaptive-inexact-proj-subgrad}, which finds an $\alpha$-minimizer without requiring knowledge of the Lipschitz constant: it adaptively chooses the step size and uses a stopping criterion based on the norms of observed subgradients. This algorithm terminates in finite time with probability \(1\) and in polynomial time with high probability because \eqref{eq:momentbound} implies that the Lipschitz parameter of \(\Phi\) is finite with probability \(1\) and polynomially bounded with high probability.

In the recursive localization procedure, the current set always has the form 
$U_j=\WW\cap\BB(c_j,r_j).$
Even if \(\WW\) is projection-friendly, such intersections need not admit an
immediate exact projection oracle. The next lemma shows that this is not an
obstacle: projection onto \(U_j\) can be implemented to arbitrary accuracy using
only logarithmically many calls to the projection oracle for \(\WW\).

\begin{lemma}[Efficient \texorpdfstring{$\xi$}{epsilon}-inexact projection onto \(\WW\cap\BB(c,r)\)]
\label{lem:inexact-proj-W0}
Let \(\WW\subseteq \mathbb{R}^d\) satisfy~\Cref{ass:main}, let \(c\in\WW\), and \(r>0\). Define $
U:=\WW\cap \mathbb{B}(c,r).$
Then, for every \(y\in\mathbb{R}^d\) and every \(\xi>0\), one can compute in polynomial time a point
\(\widetilde \Pi^\xi_{U}(y)\in U\) satisfying
\[
\bigl\|\widetilde \Pi^\xi_{U}(y)-\Pi_{U}(y)\bigr\|_2\le \xi.
\]
\end{lemma}

The proof exploits the KKT conditions for projection onto $\WW\cap \BB(w_0,r)$ to reduce the problem to a one-dimensional search over the Lagrange multiplier. We use bisection method to construct an $\xi$-inexact projector using only logarithmically many calls to the projection oracle for $\WW$.

Next, Proposition~\ref{prop:adaptive-proj-subgrad-inexact} shows that adaptive projected subgradient descent with inexact projection oracle still returns a certified \(\alpha\)-minimizer in finite time whenever the realized Lipschitz constant is finite. Thus, the method can be applied over \(U_j\), and over \(U_j^{m+1}\) by projecting each coordinate.

\begin{algorithm}[t]
\caption{\textsc{Adaptive-Inexact-ProjSubgrad}$(K,\Phi,\alpha,D)$}
\label{alg:adaptive-inexact-proj-subgrad}
\DontPrintSemicolon
\KwIn{Closed convex set \(K\subseteq\mathbb{R}^q\) with \(\operatorname{diam}(K)\le D\), convex \(\Phi\), target accuracy \(\alpha>0\)}
\KwOut{A point \(u_\alpha\in K\)}

Choose any \(x_1\in K\)\;

\For{$t=1,2,\dots$}{

Query the first-order oracle at \(x_t\) and obtain a subgradient \(g_t\in\partial \Phi(x_t)\)\;

\eIf{\(g_t=0\)}{
    \Return \(x_t\)\;
}{
    Set $
    S_t:=\sum_{s=1}^t \|g_s\|_2^2,
    \eta_t:=\frac{D}{\sqrt{S_t}},
    \xi_t:=\min\!\left\{D,\frac{\alpha\,\eta_t}{6D}\right\}
    $
    
    Compute an \(\xi_t\)-inexact projection
   $ 
    x_{t+1}:=\widetilde\Pi_K^{\xi_t}(x_t-\eta_t g_t)
    $
    
    \If{
  $
    3D\sqrt{S_t}\le \alpha t$
    }{
        \Return \(\bar x_t:=\frac1t\sum_{s=1}^t x_s\)\;
    }
}
}
\end{algorithm}

\begin{proposition}[Adaptive Inexact Projected Subgradient Method]
\label{prop:adaptive-proj-subgrad-inexact}
Let \(K \subseteq \mathbb{R}^q\) be a nonempty closed convex set equipped with an \(\xi\)-inexact projection oracle for every \(\xi>0\), and suppose
\(\operatorname{diam}(K)\le D\). Let \(\Phi:K\to\mathbb{R}\) be convex and \(L\)-Lipschitz on \(K\) for some finite but unknown \(L>0\). Suppose we are given an exact subgradient oracle for $\Phi$.
Fix \(\alpha>0\). Then, Algorithm~\ref{alg:adaptive-inexact-proj-subgrad} halts after at most
$
T_\alpha:=\left\lceil \left(\frac{3DL}{\alpha}\right)^2 \right\rceil
$
iterations, and its output \(u_\alpha\in K\) satisfies
$
\Phi(u_\alpha)-\min_{u\in K}\Phi(u)\le \alpha.
$
Hence, if each subgradient query and \(\xi_t\)-inexact projection
can be performed in polynomial time, then the total running time is polynomial in
\(q\), \(D\), \(L\), and \(1/\alpha\).
\end{proposition}

\Cref{prop:adaptive-proj-subgrad-inexact} extends standard adaptive projected subgradient method guarantees (c.f. \cite{duchi2011adaptive,shalev2012online}) to inexact projections by charging projection error into the descent estimate. It ensures that we can find an \(\alpha\)-minimizer of \(\Phi_{Z,U}\) --- which,
by \Cref{lem:localized-joint-convex}, yields an \(\alpha\)-minimizer of
\(\hat F^{(w_0)}_{C,U,\lambda,Z}\) over \(U\) --- without knowing a bound on the
realized Lipschitz constant. This is essential for pure-DP output perturbation.

\paragraph{Application to localized regularized Lipschitz extensions.}
We use \Cref{lem:localized-joint-convex} together with
\Cref{alg:adaptive-inexact-proj-subgrad} to implement the approximate
minimization step inside \Cref{alg:localized-outputpert}. Namely, whenever the
current localized set is \(U=\WW\cap\BB(c,r)\), we minimize the jointly convex objective \(\Phi_{Z,U}\) over \(U^{m+1}\) to the accuracy $\alpha$ required by Algorithm~\ref{alg:localized-outputpert}. Projection onto \(U\) is implemented using \Cref{lem:inexact-proj-W0}, with
\(w_0\) replaced by the current center \(c\). Projection onto the product set
\(U^{m+1}\) is implemented coordinatewise. 
Thus
\Cref{alg:recursive-localized-outputpert-main} can be implemented efficiently.

\section{Optimal Pure-DP Heavy-Tailed SCO}
\label{sec:main-dop}

We now combine the ingredients from
\Cref{sec:scotoERMreduction,sec:bias-shrinking,sec:joint-adaptive}. Our main
regularized ERM solver is \textsc{Recursive-Localized-OutputPert}
(\Cref{alg:recursive-localized-outputpert-main}), with each call to
\textsc{Localized-OutputPert} implemented efficiently by minimizing the localized joint
convex objective \(\Phi_{Z,U_j}\) over \(U_j^{m+1}\) using
\Cref{alg:adaptive-inexact-proj-subgrad} with diameter bound
\(\sqrt{m+1}R_j\) and the inexact projection oracle from
\Cref{lem:inexact-proj-W0}; see \Cref{alg:recursive-localized-erm} for pseudocode. Plugging the efficient ERM solver \Cref{alg:recursive-localized-erm} into line~10 of
\Cref{alg:pop-localize} gives our main algorithm: \textbf{\textsc{Pop-Recursive-Localized-OutputPert}}.

\begin{theorem}[Main theorem]
\label{thm:main-dop}
Let \(\delta, \rho\in(0,1/5)\). 
\textbf{\textsc{Pop-Recursive-Localized-OutputPert}} is \(\eps\)-differentially private and with probability at least \(1-\delta\), its output \(\hat w\) satisfies
\[
F(\hat w)-F^*
\le
c_3 G_k D
\left(\frac{d\log(1/\delta)}{n\eps}\right)^{1-\frac1k}
+
c_4 G_2 D \sqrt{\frac{\log(1/\delta)}{n}}.
\]
Moreover, its runtime is bounded with probability at least \(1-\rho\) by a polynomial in
$
n,\ d,\ D,\ \frac1\eps,\ G_k,\ \log\frac{n}{\delta},\ \frac{1}{\rho}.
$
Further, if \(\sup_{z\in\ZZ}\max_{w\in\WW}\|\nabla f(w,z)\|\)
is finite and polynomially bounded in the problem parameters, then its runtime is polynomial with probability \(1\).
\end{theorem}
The excess risk bound in \Cref{thm:main-dop} is optimal up to a factor of
\(O((\log(1/\delta))^{1-1/k})\). The runtime randomness comes only from the sample \(Z\): by \eqref{eq:momentbound}, the realized Lipschitz constant of $\Phi_{Z, U}$ is finite almost surely and polynomially bounded with high probability. By a union bound, the excess risk and runtime guarantees both hold simultaneously with probability \( \ge 1-\delta-\rho\). If the worst-case Lipschitz parameter is itself polynomially bounded, the runtime is polynomial with probability \(1\). 

Theorem~\ref{thm:main-dop} follows from~\Cref{thm:localization-generic} and \Cref{prop:erm-main}, which gives \eqref{eq:distance-to-ERM} efficiently under \(\eps\)-DP:

\begin{proposition}[Regularized ERM solver]
\label{prop:erm-main}
Fix \(w_0\in\WW\), \(\lambda>0\), and \(\rho\in(0,1/5)\).
\Cref{alg:recursive-localized-erm} is \(\eps\)-differentially private. If
\(Z\sim P^m\), then its output \(w_{\mathrm{DP}}\) satisfies
\[
\Pr\!\left(
\|w_{\mathrm{DP}}-\hat w_\lambda(Z;w_0)\|
\le
c_{\mathrm{erm}}
\frac{1}{\lambda}
\left(
G_k\left(\frac{d}{m\eps}\right)^{1-\frac1k}
+
\frac{G_2}{\sqrt m}
\right)
\right)\ge 0.7.
\]
Moreover, with probability at least \(1-\rho\), the runtime is bounded by a polynomial in
$
m,\ d,\ D,\ \lambda,\ \frac{1}{\lambda},\ \frac1\eps,\ G_k,\ \frac{1}{\rho}.
$
Further, if \(\sup_{z\in\ZZ}\max_{w\in\WW}\|\nabla f(w,z)\|\)
is finite and polynomially bounded in the problem parameters, then the runtime is polynomial with probability \(1\).
\end{proposition}

\paragraph{Proof overview.}
Privacy and the distance guarantee follow from \Cref{prop:bias-reduced-decomp}. Each approximate minimization step in \Cref{alg:localized-outputpert} is carried out by the localized joint convex reformulation and the adaptive inexact projected subgradient method. By \Cref{lem:localized-joint-convex}, \Cref{lem:inexact-proj-W0} and \Cref{prop:adaptive-proj-subgrad-inexact}, the resulting implementation has finite runtime whenever the realized Lipschitz parameter of \(\Phi_{Z,U}\) is finite, and polynomial runtime whenever this parameter is polynomially bounded. Under \eqref{eq:momentbound}, these events hold with probability \(1\) and with high probability, respectively.

\subsection{Deterministic polynomial time for structured subclasses}

For the full heavy-tailed function class, \Cref{thm:main-dop} yields optimal risk up to logarithmic factors in polynomial time with high probability, and with probability \(1\) if the worst-case Lipschitz parameter is polynomially bounded. Appendix~\ref{sec:wp1-subclasses} describes a different sufficient condition under which the same excess risk guarantee can also be obtained in deterministic polynomial time, even with infinite worst-case Lipschitz parameter: \textit{efficient approximation of localized Lipschitz-extension subgradients}. Under this condition, the approximate minimization steps in \Cref{alg:localized-outputpert} can be implemented directly by projected subgradient descent on the localized Lipschitz-extension objective, whose Lipschitz constant is deterministically bounded by \(C\).

\paragraph{Examples from Machine Learning.} Several important subclasses of convex losses admit arbitrarily accurate subgradient approximation of localized Lipschitz extensions in polynomial time under mild assumptions on $\WW$. For example: polyhedral losses on compact domains admitting an explicit second-order-cone programming representation with a strict interior point. This covers hinge/ReLU-type and absolute-value losses on Euclidean balls, ellipsoids, and polytopes. Thus, for such problems, our algorithm achieves excess risk~\eqref{eq:excriskupperbound} in deterministic polynomial time.

\section{High Probability Lower Bound}
We provide a novel high-probability excess risk lower bound, which is tighter than the expected excess risk lower bound derived from \cite{bd14} by logarithmic factors. 
\begin{theorem}[SCO lower bound]
\label{thm:sco-lb-mainbody}
Let $\delta \in (0, 1/5), \eps \le 1, 0 < G_2 \le G_k$. For every $\eps$-DP $\mathcal{A}$, there exists a problem instance $(P, f)$ satisfying \Cref{ass:main} such that
\[
\Pr_{Z\sim P^n}\!\left(
F(\mathcal A(Z))-\inf_{w\in\mathcal W}F(w)
\ge
cD\,
\min\Biggl\{
G_2,\;
G_2\sqrt{\frac{\log(1/\delta)}{n}}
+
G_k\left(\frac{d+\log(1/\delta)}{n\varepsilon}\right)^{1-1/k}
\Biggr\}
\right)\ge \delta.
\]
\end{theorem}

The trivial algorithm achieves excess risk $\le G_2 D$, so \Cref{thm:sco-lb-mainbody} is nearly tight: the only gap is that $d + \log(1/\delta)$ appears in the lower bound, while $d \log(1/\delta)$ appears in the upper bound (\Cref{thm:main-dop}). 

\section{Conclusion}
We determined the optimal rate for \(\eps\)-DP heavy-tailed SCO up to a logarithmic factor. Our algorithm achieves this rate in polynomial time with high probability. If the worst-case Lipschitz parameter is polynomially bounded, then runtime is polynomial with probability \(1\). We also identified a condition that permits deterministic polynomial time even with infinite Lipschitz parameter: efficient approximation of the Lipschitz extension subgradient, which holds for some important ML problems. 
Three natural directions remain: (1) Can the $d \log(1/\delta)$ term be improved to $d + \log(1/\delta)$ in the excess risk bound? 
(2) Is optimal $\eps$-DP risk in deterministic polynomial time possible for arbitrary losses? (3) Can our algorithm be practically implemented for ML model training?

\section*{Acknowledgements}
We thank Adam Smith for helpful early feedback on the proposed problem and Yu-Xiang Wang for pointing us to~\cite{lin2025purifying}.

\bibliographystyle{alpha}
\bibliography{references}

\clearpage

\appendix

\section*{Appendix}

\section{Zero-Concentrated Differential Privacy}
\label{app:zcdp}

We recall the standard definition of zero-concentrated differential privacy (zCDP).

\begin{definition}[zCDP]
A randomized algorithm \(\mathcal A\) is \(\rho\)-zCDP if for every pair of neighboring datasets \(Z,Z'\) and every \(\alpha>1\),
\[
D_\alpha(\mathcal A(Z)\|\mathcal A(Z')) \le \rho \alpha,
\]
where \(D_\alpha\) denotes the order-\(\alpha\) Rényi divergence.
\end{definition}

The following standard facts show that zCDP lies between pure and approximate DP.

\begin{lemma}[\cite{bun2016concentrated}]
\label{lem:zcdp-bridge}
\begin{enumerate}
    \item If \(\mathcal A\) is pure \(\eps\)-DP, then \(\mathcal A\) is \((\eps^2/2)\)-zCDP.
    \item If \(\mathcal A\) is \(\rho\)-zCDP, then for every \(\delta\in(0,1)\), \(\mathcal A\) is
    \[
    \left(\rho + 2\sqrt{\rho\log(1/\delta)},\,\delta\right)\text{-DP}.
    \]
\end{enumerate}
\end{lemma}

\section{Additional Discussion of Related Work}
\label{app:related-work}
This appendix provides additional context on related work. A concise discussion appears in the introduction; here we elaborate on the most closely related lines.

\paragraph{Private uniformly Lipschitz SCO.}
Differentially private stochastic convex optimization has been studied extensively over the past decade under uniform Lipschitz assumptions on the loss function, with sharp excess-risk guarantees now known in many settings
\cite{bft19,fkt20,AsiFeKoTa21,bassily2023differentially,lowy2024faster,lowy2025dpbilevel}. In particular, optimal rates for pure $\eps$-DP SCO under the uniform Lipschitz assumption are known up to logarithmic factors~\cite{asi2021adapting,ganesh2023universality,bst14}. This literature assumes a finite worst-case Lipschitz parameter \(L\), and the resulting upper and lower bounds scale with \(L\). In contrast, the present paper focuses on the heavy-tailed regime, where \(L\) may be arbitrarily large and only a finite \(k\)-th moment bound is assumed; our bounds scale with $G_k$ rather than $L$. 

\paragraph{Heavy-tailed private SCO.}
A recent line of work studies differentially private SCO under moment assumptions on the sample-wise Lipschitz parameters or gradients rather than a uniform Lipschitz bound
\cite{tao2021optimal,asi2021private,kamath2022improved,hu2022high,wang2022differentially,asi2024private,lowy2025private}.
These works show that one can obtain substantially sharper guarantees in heavy-tailed settings by replacing worst-case dependence on \(L\) with dependence on the moment parameter \(G_k\). In particular, the heavy-tailed regime is well understood under relaxed privacy notions such as \(\rho\)-zCDP, with efficient clipped-gradient-based algorithms achieving the optimal rates up to logarithmic factors.

\paragraph{The pure-DP gap.}
The pure \(\eps\)-DP heavy-tailed case has remained open on the algorithmic side.
Prior to this work, the main known lower-bound ingredient was the pure-DP mean-estimation lower
bound of \cite{bd14}, which implies via the standard reduction from stochastic convex optimization
to mean estimation an in-expectation / constant-probability SCO lower bound of the form
\[
\Omega\!\left(
G_k D\left(\frac{d}{n\eps}\right)^{1-\frac1k}
+
\frac{G_2D}{\sqrt n}
\right).
\]
However, no matching pure-DP upper bound for heavy-tailed SCO was previously known.
Existing heavy-tailed algorithmic approaches were based on clipped noisy gradients and were tailored
to zCDP or approximate DP, where privacy accounting under composition is significantly more
forgiving. Our paper closes this algorithmic gap by giving the first direct pure \(\eps\)-DP algorithms
matching the minimax rate up to logarithmic factors.

In addition, we prove sharper high-probability lower bounds.
For the non-private term, we use a bounded two-point construction together with the
high-probability testing framework of \cite{ma2024high}.
For the private term, we build on the pure-DP packing ideas underlying \cite{bd14}, together with
a direct reduction from quantile estimation to decoding on a packing.
This yields a high-probability private lower bound with explicit dependence on the failure
probability parameter.

\paragraph{Relation to clipped-gradient methods.}
Clipping-based methods remain central in private optimization, including in heavy-tailed settings. Our main contribution is not merely an optimal statistical upper bound in polynomial time, but a different algorithmic route: we replace clipped-gradient optimization by private optimization of a Lipschitz extension of the empirical loss.

\paragraph{Relation to black-box purification from zCDP.}
One possible black-box route to pure DP is to start from the \(\rho\)-zCDP heavy-tailed SCO algorithm of \cite{asi2024private}, convert its guarantee to approximate DP using the standard zCDP-to-approximate-DP implication, and then apply the purification framework of \cite{lin2025purifying}. However, this route does not recover our main result. The most direct composition yields a pure-DP excess risk upper bound that is \textit{looser} than our upper bound in \Cref{thm:main-dop} by $\poly(\log(nd/\eps))$ factors.
Moreover, this \textit{indirect} approach is algorithmically more cumbersome: it requires combining the population-level localization framework of \cite{asi2024private} with the purification procedure of \cite{lin2025purifying}, thereby passing through multiple intermediate privacy notions.
In contrast, our algorithm is designed directly for pure DP and achieves a high-probability excess-risk guarantee without passing through an intermediate zCDP-to-approximate-DP conversion or an additional purification step, and eliminating spurious polylogarithmic factors.

\paragraph{Lipschitz extensions in optimization and privacy.}
Lipschitz extensions are classical objects in convex analysis \cite{hiriart2013convex} and have appeared in several privacy-related contexts. In our setting, the challenge is not only to use the extension as an analytical device, but to optimize it efficiently under pure DP. This requires handling the fact that exact evaluation of the Lipschitz extension cannot be certified from a fixed finite collection of local oracle queries in general (Appendix~\ref{app:impossible}) and that pure-DP output perturbation requires deterministic optimization accuracy guarantees. Our jointly convex reformulation and adaptive certified projected subgradient framework are designed to address exactly this obstacle.

\paragraph{Output perturbation and localization.}
Output perturbation for strongly convex empirical risk minimization goes back at least to \cite{chaud,bst14,lowy2021output}, and localization ideas play an important role in modern private optimization. Our use of output perturbation is structurally related to these earlier works, but serves a different purpose: we use recursive output-perturbation localization steps to shrink the domain so
that the localized Lipschitz-extension bias becomes small enough. The final
private center is then used as the phasewise regularized ERM output.

\paragraph{Summary of positioning.}
In summary, prior work resolved the heavy-tailed SCO problem under zCDP and established a pure-DP lower bound, but did not provide a matching pure-DP upper bound. 
This paper is, to the best of our knowledge, the first to obtain a direct polynomial-time pure-DP algorithm for heavy-tailed SCO with a nearly-tight high-probability excess-risk guarantee.
In particular, unlike the black-box purification route from zCDP to approximate DP and then purifying, our guarantee does not incur the additional conversion and polylogarithmic overheads discussed above. Moreover, we establish novel high-probability lower bounds by combining pure-DP packing ideas with high-probability testing and decoding arguments.

\begin{remark}[On the assumption that \(G_k\) is known]
\label{rem:public}
We assume throughout that \(G_k\) is known, and our algorithms tune the Lipschitz-extension parameter \(C\) as a function of \(G_k\). This assumption is standard in the heavy-tailed private optimization literature. Parameter-free DP SCO is an interesting direction for future work. Our focus in this work is to determine the minimax excess-risk rate for pure DP under the standard moment-bounded model.
\end{remark}

\section{Finite Local Information Does Not Determine the Lipschitz Extension}
\label{app:impossible}

The Lipschitz extension
\[
f_C(w)=\inf_{y\in\mathcal W}\{f(y)+C\|w-y\|\}
\]
is a global object. The following oracle lower-bound example shows that, even
for convex \(C^\infty\) functions, a fixed finite set of local queries does not
determine the exact value of \(f_C(w)\). This motivates our use of certified
approximate optimization of a joint convex reformulation rather than exact
evaluation of the extension.

\begin{theorem}[Finite nonadaptive local information does not determine the Lipschitz extension]
\label{thm:interior_exact_extension_impossible}
Fix \(k\ge 2\), \(0<C<G_k\), and an integer \(T\ge 1\). Let
\(q_1,\dots,q_T\in B_d:=\{y\in\mathbb R^d:\|y\|\le 1\}\) be any fixed query
points, where \(d\ge T+1\). Then there exist two convex \(C^\infty\) functions
\(f_0,f_1:B_d\to\mathbb R\) such that
\[
\max_{y\in B_d}\|\nabla f_i(y)\|\le G_k
\qquad\text{for }i\in\{0,1\},
\]
their complete local information agrees at every query point,
\[
\nabla^m f_0(q_t)=\nabla^m f_1(q_t)
\qquad
\text{for every }t\in[T]\text{ and every }m\ge 0,
\]
where \(m=0\) denotes equality of function values, but their Lipschitz-extension
values at the interior point \(w=0\) are different:
\[
(f_0)_C(0)\neq (f_1)_C(0),
\qquad
(f_i)_C(0):=\inf_{y\in B_d}\bigl\{f_i(y)+C\|y\|\bigr\}.
\]
Consequently, exact evaluation of \(f_C(0)\) cannot be determined, in general,
from any fixed finite collection of local oracle queries.
\end{theorem}

\begin{proof}
\smallskip
\noindent\textbf{Step 1: A smooth convex function flat at the origin.}
Define
\[
\eta(s):=
\begin{cases}
e^{-1/s^2}, & s\neq 0,\\[1mm]
0, & s=0.
\end{cases}
\]
It is standard that \(\eta\in C^\infty(\mathbb R)\), \(\eta(s)\ge 0\), and
\[
\eta^{(m)}(0)=0\qquad \text{for all }m\ge 0.
\]
Now set
\[
\varphi(s):=\int_0^s\int_0^u \eta(r)\,dr\,du.
\]
Then \(\varphi\in C^\infty(\mathbb R)\), \(\varphi''(s)=\eta(s)\ge 0\), so
\(\varphi\) is convex, and
\[
\varphi^{(m)}(0)=0\qquad\text{for all }m\ge 0.
\]
Moreover, \(\varphi(s)>0\) for every \(s\neq 0\), and
\[
\varphi(s)=o(s)\qquad \text{as }s\to 0.
\]
Let
\[
M:=\sup_{|s|\le 1} |\varphi'(s)| < \infty.
\]

\smallskip
\noindent\textbf{Step 2: Define two one-dimensional convex profiles with identical jets at \(0\).}
Choose any \(a\in(C,G_k)\). Since \(a<G_k\), we may choose
\(0<\beta_0<\beta_1\) such that
\[
a+\beta_1 M\le G_k.
\]
For \(i\in\{0,1\}\), define
\[
g_i(t):= -a t + \beta_i \varphi(t), \qquad t\in[-1,1].
\]
Each \(g_i\) is convex and \(C^\infty\). Since every derivative of \(\varphi\)
vanishes at \(0\),
\[
g_0^{(m)}(0)=g_1^{(m)}(0)\qquad \text{for all }m\ge 0.
\]

\smallskip
\noindent\textbf{Step 3: Lift to high dimension along a direction hidden from the fixed queries.}
Since \(d\ge T+1\), there exists a unit vector \(v\in\mathbb R^d\) orthogonal to
all query points:
\[
\langle v,q_t\rangle=0
\qquad\text{for }t=1,\dots,T.
\]
Define
\[
f_i(y):=g_i(\langle v,y\rangle),
\qquad y\in B_d,\quad i\in\{0,1\}.
\]
Since \(g_i\) is convex and \(y\mapsto \langle v,y\rangle\) is linear, each
\(f_i\) is convex and \(C^\infty\). Moreover,
\[
\nabla f_i(y)=g_i'(\langle v,y\rangle)v
=
\bigl(-a+\beta_i\varphi'(\langle v,y\rangle)\bigr)v,
\]
so for every \(y\in B_d\),
\[
\|\nabla f_i(y)\|
\le a+\beta_i\sup_{|s|\le1}|\varphi'(s)|
\le a+\beta_1M
\le G_k.
\]

At each query point \(q_t\), we have \(\langle v,q_t\rangle=0\). Therefore,
for \(m=0\),
\[
f_0(q_t)=g_0(0)=g_1(0)=f_1(q_t).
\]
For every integer \(m\ge1\),
\[
\nabla^m f_i(q_t)=g_i^{(m)}(0)\,v^{\otimes m}.
\]
Since \(g_0^{(m)}(0)=g_1^{(m)}(0)\) for all \(m\ge1\), the complete local
information of \(f_0\) and \(f_1\) agrees at every \(q_t\).

\smallskip
\noindent\textbf{Step 4: The extension values at \(0\) differ.}
Write any \(y\in B_d\) as
\[
y=t v+u,
\qquad \langle u,v\rangle=0,
\qquad t^2+\|u\|^2\le1.
\]
Then
\[
f_i(y)=g_i(t),
\qquad
\|y\|=\sqrt{t^2+\|u\|^2}\ge |t|,
\]
with equality when \(u=0\). Hence
\[
(f_i)_C(0)
=
\inf_{\|y\|\le1}\{f_i(y)+C\|y\|\}
=
\inf_{t\in[-1,1]}\{g_i(t)+C|t|\}.
\]
Define
\[
h_i(t):=g_i(t)+C|t|.
\]
For \(t<0\),
\[
h_i(t)
=
-a t+\beta_i\varphi(t)+C|t|
=
(a+C)|t|+\beta_i\varphi(t)
>0,
\]
and \(h_i(0)=0\). For \(t>0\),
\[
h_i(t)=-(a-C)t+\beta_i\varphi(t).
\]
Since \(a>C\) and \(\varphi(t)=o(t)\) as \(t\downarrow0\), there exists
\(t_i>0\) such that \(h_i(t_i)<0\). Thus every minimizer of \(h_i\) over
\([-1,1]\) lies in \((0,1]\).

Finally, for every \(t>0\),
\[
h_1(t)
=
h_0(t)+(\beta_1-\beta_0)\varphi(t)
>
h_0(t).
\]
Since both minima are attained at positive points, it follows that
\[
(f_1)_C(0)=\min_{t\in[-1,1]}h_1(t)
>
\min_{t\in[-1,1]}h_0(t)
=
(f_0)_C(0).
\]
This proves the theorem.
\end{proof}

\section{Deferred Proofs for \Cref{sec:building-blocks}}

\subsection{Proofs for \Cref{sec:lipext-building-block}}
\label{app:lipext-building}

For a dataset \(Z=(z_1,\dots,z_m)\), define
\[
\hat F_Z(w):=\frac1m\sum_{i=1}^m f(w,z_i),
\qquad
\hat F_{C,Z}(w):=\frac1m\sum_{i=1}^m f_C(w,z_i).
\]

\begin{lemma}[Bias of the Empirical Lipschitz Extension]
\label{lem:empLipExtBias-precise}
Assume \(f(\cdot,z)\) is convex on a domain \(\WW\) of diameter \(D\). Denote $a_+ := \max(a, 0).$ Then for every dataset
\(Z=(z_1,\dots,z_m)\) and every \(w\in\WW\),
\[
\hat F_Z(w)-\hat F_{C,Z}(w)
\le
\frac{D}{m}\sum_{i=1}^m (A(z_i)-C)_+,
\qquad
A(z):=\sup_{u\in\WW}\sup_{g\in\partial f(\cdot,z)(u)}\|g\|.
\]
Consequently, under the moment condition \eqref{eq:momentbound},
\[
\expec_Z\!\left[\max_{w\in\WW}\bigl(\hat F_Z(w)-\hat F_{C,Z}(w)\bigr)\right]
\le
\frac{DG_k^k}{(k-1)C^{k-1}}.
\]
\end{lemma}

\begin{proof}
Fix \(z\) and define
\[
A(z):=\sup_{u\in\WW}\sup_{g\in\partial f(\cdot,z)(u)}\|g\|.
\]
Since \(f(\cdot,z)\) is convex on the convex set \(\WW\), it is \(A(z)\)-Lipschitz on \(\WW\). Hence for every \(w,y\in\WW\),
\[
|f(w,z)-f(y,z)|\le A(z)\|w-y\|.
\]
By definition of the Lipschitz extension,
\[
f_C(w,z)=\inf_{y\in\WW}\bigl[f(y,z)+C\|w-y\|\bigr],
\]
so
\begin{align*}
f(w,z)-f_C(w,z)
&=
\sup_{y\in\WW}\bigl(f(w,z)-f(y,z)-C\|w-y\|\bigr) \\
&\le
\sup_{y\in\WW}(A(z)-C)\|w-y\| \\
&\le
D(A(z)-C)_+.
\end{align*}
Averaging over \(z_1,\dots,z_m\) gives the first claim.

Taking expectation over \(Z\) and using i.i.d. sampling,
\[
\expec_Z\!\left[\max_{w\in\WW}\bigl(\hat F_Z(w)-\hat F_{C,Z}(w)\bigr)\right]
\le
D\,\expec_{z\sim P}[(A(z)-C)_+].
\]
Using the layer-cake representation and Markov's inequality,
\begin{align*}
\expec[(A(z)-C)_+]
&=
\int_C^\infty \pr(A(z)\ge s)\,ds \\
&\le
\int_C^\infty \frac{\expec[A(z)^k]}{s^k}\,ds
=
\frac{G_k^k}{(k-1)C^{k-1}}.
\qedhere
\end{align*}
\end{proof}

\subsection{Details and proofs for \Cref{sec:bias-shrinking}}
\label{app:bias-shrinking}

 For a dataset \(Z=(z_1,\dots,z_m)\), define, for \(w\in U\),
\[
\hat F_Z(w)
:=
\frac1m\sum_{i=1}^m f(w,z_i),
\qquad
\hat F_{C,U,Z}(w)
:=
\frac1m\sum_{i=1}^m f_{C,U}(w,z_i),
\]
and
\[
\hat F^{(w_0)}_{\lambda,Z}(w)
:=
\hat F_Z(w)+\frac{\lambda}{2}\|w-w_0\|^2,
\qquad
\hat F^{(w_0)}_{C,U,\lambda,Z}(w)
:=
\hat F_{C,U,Z}(w)+\frac{\lambda}{2}\|w-w_0\|^2.
\]
Finally, define
\[
\hat w_\lambda(Z;w_0)
:=
\argmin_{w\in\WW}
\hat F^{(w_0)}_{\lambda,Z}(w),
\qquad
\hat w_{C,U,\lambda}(Z;w_0)
:=
\argmin_{w\in U}
\hat F^{(w_0)}_{C,U,\lambda,Z}(w).
\]

\begin{lemma}[Precise Statement of \Cref{lem:localized-lipext-bias}]
Let \(R>0\), and define
\[
A(z):=\sup_{u\in\WW}\sup_{g\in\partial f(\cdot,z)(u)}\|g\|.
\]
Then, for every dataset \(Z=(z_1,\dots,z_m)\), simultaneously for every closed
convex set \(U\subseteq\WW\) with \(\diam(U)\le R\),
\[
\max_{w\in U}
\bigl(\hat F_Z(w)-\hat F_{C,U,Z}(w)\bigr)
\le
R\cdot \frac1m\sum_{i=1}^m (A(z_i)-C)_+ .
\]
Consequently,
\[
\mathbb E_Z\!\left[
\sup_{\substack{U\subseteq\WW\ \mathrm{closed,convex}\\ \diam(U)\le R}}
\max_{w\in U}
\bigl(\hat F_Z(w)-\hat F_{C,U,Z}(w)\bigr)
\right]
\le
\frac{R G_k^k}{(k-1)C^{k-1}}.
\]
and for every \(\beta\in(0,1)\), with probability at least \(1-\beta\), simultaneously for every closed convex \(U\subseteq\WW\) with \(\diam(U)\le R\) 
\[
\max_{w\in U}
\bigl(\hat F_Z(w)-\hat F_{C,U,Z}(w)\bigr)
\le
\frac{R G_k^k}{\beta(k-1)C^{k-1}}.
\] 
Moreover, with probability at least \(1-\beta\), the following holds
simultaneously for every closed convex \(U\subseteq\WW\) with
\(\diam(U)\le R\) and \(\hat w_\lambda(Z;w_0)\in U\):
\[
\|\hat w_{C,U,\lambda}(Z;w_0)-\hat w_\lambda(Z;w_0)\|
\le
\sqrt{
\frac{2R G_k^k}{\lambda\beta(k-1)C^{k-1}}
}.
\]
\end{lemma}
\begin{proof}
Fix \(z\) and a closed convex set \(U\subseteq\WW\). For every \(w,y\in U\),
Lipschitz continuity implies
\[
f(w,z)-f(y,z)\le A(z)\|w-y\|.
\]
Thus, for every \(w\in U\),
\begin{align*}
f(w,z)-f_{C,U}(w,z)
&=
\sup_{y\in U}
\bigl(f(w,z)-f(y,z)-C\|w-y\|\bigr) \\
&\le
\sup_{y\in U}(A(z)-C)\|w-y\| \\
&\le
\diam(U)(A(z)-C)_+.
\end{align*}
Therefore, if \(\diam(U)\le R\), then
\[
\max_{w\in U}
\bigl(\hat F_Z(w)-\hat F_{C,U,Z}(w)\bigr)
\le
R\cdot \frac1m\sum_{i=1}^m (A(z_i)-C)_+.
\]
This deterministic inequality holds simultaneously for all such \(U\).

Taking expectation and using
\[
\mathbb E[(A(z)-C)_+]
\le
\int_C^\infty \frac{\mathbb E[A(z)^k]}{s^k}\,ds
\le
\frac{G_k^k}{(k-1)C^{k-1}}
\]
proves the expectation bound. The high-probability bound follows from Markov's
inequality applied to \(\frac1m\sum_{i=1}^m(A(z_i)-C)_+\). 

It remains to prove the distance bound. The distance statement is a consequence of the preceding high-probability event.
Indeed, on that event, fix any closed convex \(U\subseteq\WW\) with
\(\diam(U)\le R\) and \(\hat w_\lambda(Z;w_0)\in U\), and let
\[
B:=\frac{R G_k^k}{\beta(k-1)C^{k-1}}.
\] 
On the high-probability event above, suppose \(\hat w_\lambda(Z;w_0)\in U\).
Since \(f_{C,U}\le f\) on \(U\), we have
\[
\hat F^{(w_0)}_{C,U,\lambda,Z}(\hat w_\lambda)
\le
\hat F^{(w_0)}_{\lambda,Z}(\hat w_\lambda).
\]
By optimality of \(\hat w_{C,U,\lambda}\) over \(U\),
\[
\hat F^{(w_0)}_{C,U,\lambda,Z}(\hat w_{C,U,\lambda})
\le
\hat F^{(w_0)}_{C,U,\lambda,Z}(\hat w_\lambda).
\]
Finally, by the localized bias bound applied at
\(\hat w_{C,U,\lambda}\in U\),
\[
\hat F^{(w_0)}_{\lambda,Z}(\hat w_{C,U,\lambda})
\le
\hat F^{(w_0)}_{C,U,\lambda,Z}(\hat w_{C,U,\lambda})+B.
\]
Combining the last three displays gives
\[
\hat F^{(w_0)}_{\lambda,Z}(\hat w_{C,U,\lambda})
-
\hat F^{(w_0)}_{\lambda,Z}(\hat w_\lambda)
\le B.
\]
Since \(\hat F^{(w_0)}_{\lambda,Z}\) is \(\lambda\)-strongly convex and minimized
at \(\hat w_\lambda\),
\[
\frac{\lambda}{2}
\|\hat w_{C,U,\lambda}-\hat w_\lambda\|^2
\le B,
\]
which proves the claimed distance bound.
\end{proof}

\paragraph{Localizing the domain via output perturbation.}
Given a current set \(U\), \Cref{alg:localized-outputpert} privately computes
an approximate minimizer of the \(U\)-localized Lipschitz-extension objective,
adds Laplace noise, and returns a ball around the noisy point. The radius is
chosen so that, if \(U\) contains the regularized ERM minimizer
\(\hat w_\lambda(Z;w_0)\), the returned set also contains it with high
probability.

\begin{lemma}[Guarantees of \Cref{alg:localized-outputpert}]
\label{lem:localized-outputpert}
Let \(Z\sim P^m\), \(w_0\in\WW\), and \(\lambda>0\).
For every fixed closed convex set \(U\subseteq\WW\) and deterministic
\(R\ge \diam(U)\), \Cref{alg:localized-outputpert} is \(\eps\)-DP. Moreover,
the utility statement remains valid when \(U\) is chosen adaptively from
previous private outputs, provided \(R\ge\diam(U)\) is the deterministic
diameter bound used by the algorithm: if
\(\hat w_\lambda(Z;w_0)\in U\), then with probability at least \(1-2\beta\),
\[
\hat w_\lambda(Z;w_0)\in U^+.
\]
\end{lemma}

\begin{proof}
Let
\[
\hat w_{C,U,\lambda}(Z;w_0)
:=
\argmin_{w\in U}\hat F^{(w_0)}_{C,U,\lambda,Z}(w).
\]

\paragraph{Privacy.}
For fixed \(U\), each \(f_{C,U}(\cdot,z)\) is \(C\)-Lipschitz on \(U\). Hence
exact minimizers of neighboring \(\lambda\)-strongly convex empirical objectives
have \(\ell_2\)-sensitivity at most \(2C/(\lambda m)\). Since
\(\tilde w\) is an \(\alpha\)-approximate minimizer and
\(\hat F^{(w_0)}_{C,U,\lambda,Z}\) is \(\lambda\)-strongly convex,
\[
\|\tilde w-\hat w_{C,U,\lambda}(Z;w_0)\|
\le
\sqrt{\frac{2\alpha}{\lambda}}
=
\frac{C}{\sqrt{50}\lambda m}
\min\left\{1,\frac{d\ell}{\eps}\right\}
\le
\frac{C}{\sqrt{50}\lambda m}.
\]
Therefore the $\ell_2$-sensitivity of the deterministic map \(Z\mapsto\tilde w(Z)\) is
at most
\[
\frac{2C}{\lambda m}
+
2\cdot \frac{C}{\sqrt{50}\lambda m}
\le
\frac{3C}{\lambda m}.
\]
The isotropic Laplace mechanism with density proportional to
\[
\exp\!\left(
-\frac{\eps\lambda m}{10C}\|b\|
\right)
\]
is therefore \(\eps\)-DP. Projection onto \(\WW\)
and the construction of \(U^+\) are post-processing. For adaptively chosen
\(U\), this privacy statement applies conditionally on the previous private
outputs, and the full recursive use follows by adaptive composition.

\paragraph{Uniform localized-bias event.}
Define
\[
A(z):=\sup_{u\in\WW}\sup_{g\in\partial f(\cdot,z)(u)}\|g\|
\]
and the event
\[
E_{\mathrm{bias}}
:=
\left\{
\frac1m\sum_{i=1}^m (A(z_i)-C)_+
\le
\frac{G_k^k}{\beta(k-1)C^{k-1}}
\right\}.
\]
By the tail-integral bound and Markov's inequality,
\(\Pr(E_{\mathrm{bias}})\ge 1-\beta\). On \(E_{\mathrm{bias}}\), by
\Cref{lem:localized-lipext-bias}, simultaneously for every closed convex
\(V\subseteq\WW\) with \(\diam(V)\le R\),
\[
\max_{w\in V}
\bigl(\hat F_Z(w)-\hat F_{C,V,Z}(w)\bigr)
\le
\frac{R G_k^k}{\beta(k-1)C^{k-1}}.
\]
This uniformity is what allows the utility proof to apply to adaptively chosen
sets \(U\).

\paragraph{Distance from the localized-extension minimizer to the original ERM minimizer.}
Assume \(\hat w_\lambda(Z;w_0)\in U\). On \(E_{\mathrm{bias}}\), using
\(\diam(U)\le R\),
\[
\max_{w\in U}
\bigl(\hat F_Z(w)-\hat F_{C,U,Z}(w)\bigr)
\le
\frac{R G_k^k}{\beta(k-1)C^{k-1}}.
\]
For every \(w\in U\),
\[
\hat F^{(w_0)}_{\lambda,Z}(w)
\le
\hat F^{(w_0)}_{C,U,\lambda,Z}(w)
+
\max_{u\in U}\bigl(\hat F_Z(u)-\hat F_{C,U,Z}(u)\bigr).
\]
Moreover, since \(\hat w_\lambda\in U\) and \(f_{C,U}\le f\) on \(U\),
\[
\hat F^{(w_0)}_{C,U,\lambda,Z}(\hat w_{C,U,\lambda})
\le
\hat F^{(w_0)}_{C,U,\lambda,Z}(\hat w_\lambda)
\le
\hat F^{(w_0)}_{\lambda,Z}(\hat w_\lambda).
\]
Combining the last two displays gives
\[
\hat F^{(w_0)}_{\lambda,Z}(\hat w_{C,U,\lambda})
-
\hat F^{(w_0)}_{\lambda,Z}(\hat w_\lambda)
\le
\frac{R G_k^k}{\beta(k-1)C^{k-1}}.
\]
Since \(\hat F^{(w_0)}_{\lambda,Z}\) is \(\lambda\)-strongly convex,
\[
\|\hat w_{C,U,\lambda}-\hat w_\lambda\|
\le
\sqrt{
\frac{2R G_k^k}{\lambda\beta(k-1)C^{k-1}}
}
\]
on \(E_{\mathrm{bias}}\). 

\paragraph{Noise tail and conclusion.}
Let
\[
s:=\frac{10C}{\eps\lambda m},
\qquad
\ell:=\log(1/\beta).
\]
For isotropic Laplace noise with density proportional to \(\exp(-\|b\|/s)\),
the radial variable \(\|b\|/s\) has Gamma\((d,1)\) distribution. Hence by a Chernoff bound
\[
\Pr\left(
\|b\|
\le
12\frac{Cd\ell}{\lambda m\eps}
\right)
\ge
1-\beta.
\]
Call this event \(E_{\mathrm{noise}}\).

On \(E_{\mathrm{bias}}\cap E_{\mathrm{noise}}\), using nonexpansiveness of
projection onto \(\WW\) and \(\hat w_\lambda\in\WW\),
\begin{align*}
\|w_{\mathrm{loc}}-\hat w_\lambda\|
&=
\|\Pi_{\WW}(\tilde w+b)-\Pi_{\WW}(\hat w_\lambda)\| \\
&\le
\|\tilde w+b-\hat w_\lambda\| \\
&\le
\|\tilde w-\hat w_{C,U,\lambda}\|
+
\|\hat w_{C,U,\lambda}-\hat w_\lambda\|
+
\|b\| \\
&\le
\sqrt{
\frac{2R G_k^k}{\lambda\beta(k-1)C^{k-1}}
}
+
\frac{20Cd\ell}{\lambda m\eps} \le r^{+}
\end{align*}
Here we used the choice of \(\alpha\), which gives
\[
\|\tilde w-\hat w_{C,U,\lambda}\|
\le
\frac{C}{\sqrt{50}\lambda m}
\min\left\{1,\frac{d\ell}{\eps}\right\}
\le
\frac{Cd\ell}{\sqrt{50}\lambda m\eps}.
\]

Therefore
\[
\hat w_\lambda(Z;w_0)\in
\WW\cap\BB(w_{\mathrm{loc}},R^+)
=
U^+,
\]
on $E_{\mathrm{bias}}\cap E_{\mathrm{noise}}$.

Finally,
\[
\Pr(E_{\mathrm{bias}}\cap E_{\mathrm{noise}})
\ge
1-2\beta,
\]
which proves the lemma.
\end{proof}

\begin{proof}[Proof of \Cref{prop:bias-reduced-decomp}]
\textbf{Privacy:} Privacy follows from basic adaptive composition:
\[
\sum_{j=0}^{S-1}\eps_j
=
\frac{\eps}{2}\sum_{j=0}^{S-1}\theta_j
=
\frac{\eps}{2}
\le \eps.
\]
Thus the recursive localization algorithm is \(\eps\)-DP.

\paragraph{Distance Bound:}
For each round \(j\), define the bias-good event
\[
B_j
:=
\left\{
\frac1m\sum_{i=1}^m (A(z_i)-C_j)_+
\le
\frac{G_k^k}{\beta_j(k-1)C_j^{k-1}}
\right\},
\qquad
A(z):=\sup_{u\in\WW}\sup_{g\in\partial f(\cdot,z)(u)}\|g\|.
\]
By Markov's inequality and \Cref{lem:localized-lipext-bias},
\(\Pr(B_j)\ge1-\beta_j\). Crucially, \(B_j\) controls the localized-extension
bias simultaneously for every closed convex \(U\subseteq\WW\) with
\(\diam(U)\le R_j\), so it applies to the adaptively chosen set \(U_j\).

Let \[ 
N_j := \left\{\|b_j\| \le c_{\mathrm{Lap}}\frac{C_j d \ell_j}{\lambda m \eps_j} \right\}
\]
denote the corresponding Laplace-noise good event in round \(j\), where $b_j$ is the isotropic Laplace noise that is added in the $j$-th call of \Cref{alg:localized-outputpert} and $c_{\mathrm{Lap}}$ is a sufficiently large absolute constant that we will absorb into the final constant.  
Then \(\Pr(N_j)\ge1-\beta_j\). On the event
\[
\mathcal G:=\bigcap_{j=0}^{S-1}(B_j\cap N_j),
\]
an induction using \Cref{lem:localized-outputpert} gives
\[
\hat w_\lambda(Z;w_0)\in U_j
\qquad\text{for all }j=0,\dots,S.
\]
Moreover,
\[
\Pr(\mathcal G)
\ge
1-2\sum_{j=0}^{S-1}\beta_j
=
1-2\beta_0
\ge 0.98.
\]

It remains to bound the final radius. Set
\[
a:=\frac1{k+1},
\qquad
b:=\frac{k-1}{k+1},
\qquad
R_\star
:=
\frac{G_k}{\lambda}
\left(\frac{d}{m\eps}\right)^{1-\frac1k}.
\]
If \(R_\star\ge D\), then
\[
\|w_{\mathrm{loc},S}-\hat w_\lambda(Z;w_0)\|\le D\le R_\star,
\]
so the claim holds. Hence assume \(R_\star<D\).

By the radius update and the definitions of \(\eps_j,\beta_j\),
\[
R_{j+1}
\le
c
\left(\frac{G_k}{\lambda}\right)^{k/(k+1)}
R_j^a
\left(\frac{d\ell_j}{m\eps_j}\right)^b
\beta_j^{-a}.
\] 
 
Dividing by \(R_\star\), and using
\[
\eps_j=\frac{\eps}{2}\theta_j,
\qquad
\beta_j=\beta_0\theta_j,
\]
we obtain
\[
\frac{R_{j+1}}{R_\star}
\le
C_k
\left(\frac{R_j}{R_\star}\right)^a
\theta_j^{-(a+b)}
\ell_j^b,
\]
where \(C_k\) is an absolute constant depending only on \(k\) and \(\beta_0\).
Here \(a+b=k/(k+1)\). Indeed, the identity follows from
\[
R_\star^{1-a}
=
\left(\frac{G_k}{\lambda}\right)^{k/(k+1)}
\left(\frac{d}{m\eps}\right)^{(k-1)/(k+1)}.
\]

Unrolling the recurrence gives
\[
\frac{R_S}{R_\star}
\le
C_k^{1+a+\cdots+a^{S-1}}
\left(\frac{R_0}{R_\star}\right)^{a^S}
\prod_{j=0}^{S-1}
\left(
\theta_j^{-(a+b)}\ell_j^b
\right)^{a^{S-1-j}}.
\]
Since \(a<1\), the prefactor \(C_k^{1+a+\cdots+a^{S-1}}\) is bounded by
a constant depending only on \(k\), and we absorb it into \(C_k\). Thus
\[
\frac{R_S}{R_\star}
\le
C_k
\left(\frac{R_0}{R_\star}\right)^{a^S}
\prod_{j=0}^{S-1}
\left(
\theta_j^{-(a+b)}\ell_j^b
\right)^{a^{S-1-j}}.
\]

It remains to bound the product. Write \(q=S-1-j\). By the choice
\[
\theta_j
=
\frac{(1-a)a^{S-1-j}}{1-a^S},
\]
we have
\[
\theta_j \asymp_k a^q,
\qquad
\ell_j=\log(1/\beta_j)
=
\log(1/(\beta_0\theta_j))
\lesssim_k 1+q .
\]
Therefore
\[
\log
\prod_{j=0}^{S-1}
\left(
\theta_j^{-(a+b)}\ell_j^b
\right)^{a^{S-1-j}}
=
\sum_{q=0}^{S-1}
a^q
\left[
(a+b)\log(1/\theta_{S-1-q})
+
b\log \ell_{S-1-q}
\right],
\]
by changing variables $q = S - 1 - j$. 
Using \(\theta_{S-1-q}\asymp_k a^q\) and \(\ell_{S-1-q}\lesssim_k 1+q\), this is at most
\[
C_k
+
C_k\sum_{q=0}^{\infty}a^q q
+
C_k\sum_{q=0}^{\infty}a^q\log(1+q)
\le C_k
\]
after enlarging $C_k$,
because \(a\in(0,1)\). Hence the product is bounded by a constant depending
only on \(k\) and \(\beta_0\). Consequently,
\[
R_S
\le
C_kR_\star
\left(\frac{R_0}{R_\star}\right)^{a^S}.
\]

Since \(R_0=D\) and
\[
S
=
\left\lceil
\log_{k+1}\log\left(\frac{eD}{R_\star}\right)
\right\rceil,
\]
we have
\[
\left(\frac{D}{R_\star}\right)^{a^S}
\le e.
\]
Hence
\[
R_S\le C_k R_\star
\]
after enlarging the constant $C_k$. 

On \(\mathcal G\), the final center \(w_{\mathrm{loc},S}\) satisfies
\[
\|w_{\mathrm{loc},S}-\hat w_\lambda(Z;w_0)\|\le R_S
\le
C_k
\frac{G_k}{\lambda}
\left(\frac{d}{m\eps}\right)^{1-\frac1k}.
\]
This proves the proposition.
\end{proof}

\subsection{Proofs for \Cref{sec:joint-adaptive}}
\label{app:joint-adaptive}

\begin{lemma}[Re-statement of \Cref{lem:localized-joint-convex}]
\label{lem:localized-joint-convex-precise}
The function \(\Phi_{Z,U}\) is convex on \(U^{m+1}\), and
\[
\min_{(w,y_1,\dots,y_m)\in U^{m+1}}\Phi_{Z,U}(w,y_1,\dots,y_m)
=
\min_{w\in U}\hat F^{(w_0)}_{C,U,\lambda,Z}(w).
\]
Moreover, if
\[
u_\alpha=(w_\alpha,y_{1,\alpha},\dots,y_{m,\alpha})\in U^{m+1}
\]
satisfies
\[
\Phi_{Z,U}(u_\alpha)-\min_{u\in U^{m+1}}\Phi_{Z,U}(u)\le \alpha,
\]
then
\[
\hat F^{(w_0)}_{C,U,\lambda,Z}(w_\alpha)
-
\min_{w\in U}\hat F^{(w_0)}_{C,U,\lambda,Z}(w)
\le \alpha.
\]
\end{lemma}

\begin{proof}
Convexity is immediate since each \(f(\cdot,z_i)\) is convex and
\((w,y_i)\mapsto \|w-y_i\|\) is jointly convex. For fixed \(w\in U\), the
variables \(y_1,\dots,y_m\) separate, and
\[
\min_{y_i\in U}\bigl[f(y_i,z_i)+C\|w-y_i\|\bigr]=f_{C,U}(w,z_i).
\]
Substituting this identity into \eqref{eq:localized-joint-program} proves the equality of optima. The final claim follows from
\[
\hat F^{(w_0)}_{C,U,\lambda,Z}(w_\alpha)
=
\min_{y_1,\dots,y_m\in U}\Phi_{Z,U}(w_\alpha,y_1,\dots,y_m)
\le
\Phi_{Z,U}(u_\alpha).
\qedhere
\]
\end{proof}

\begin{proposition}[Precise version of \Cref{prop:adaptive-proj-subgrad-inexact}]
\label{prop:adaptive-proj-subgrad-inexact-precise}
Let \(K \subseteq \mathbb{R}^q\) be a nonempty compact convex set equipped with an \(\xi\)-inexact projection oracle for every \(\xi>0\), and suppose
\(\operatorname{diam}(K)\le D\). Let \(\Phi:K\to\mathbb{R}\) be convex and \(L\)-Lipschitz on \(K\) for some finite but unknown \(L>0\). Suppose we are given an exact subgradient oracle: for every \(x\in K\), the oracle returns a vector \(g(x)\in\mathbb{R}^q\) satisfying
\[
\Phi(y)\ge \Phi(x)+\langle g(x),y-x\rangle
\qquad\text{for all }y\in K.
\]

Fix \(\alpha>0\). Then Algorithm~\ref{alg:adaptive-inexact-proj-subgrad} halts after at most
\[
T_\alpha:=\left\lceil \left(\frac{3DL}{\alpha}\right)^2 \right\rceil
\]
iterations, and its output \(u_\alpha\in K\) satisfies
\[
\Phi(u_\alpha)-\min_{u\in K}\Phi(u)\le \alpha.
\]

In particular, the algorithm uses at most \(T_\alpha\) calls to the subgradient oracle and at most \(T_\alpha\) calls to the inexact projection oracle. Hence, if each subgradient query and each \(\xi\)-inexact projection can be performed in polynomial time, then the total running time is polynomial in \(q\), \(D\), \(L\), and \(1/\alpha\).
\end{proposition}

\begin{proof}
Let \(x^\star\in\arg\min_{x\in K}\Phi(x)\), which exists since \(K\) is compact and \(\Phi\) is continuous.

If the algorithm encounters \(g_t=0\), then \(x_t\) is an exact minimizer and the claim is immediate. Hence assume \(g_t\neq 0\) before termination.

Fix \(t\ge 1\), and let
\[
y_t:=x_t-\eta_t g_t,
\qquad
p_t:=\Pi_K(y_t),
\qquad
e_t:=2D\xi_t+\xi_t^2.
\]
Since \(\|x_{t+1}-p_t\|\le \xi_t\) and \(\operatorname{diam}(K)\le D\),
\[
\|x_{t+1}-x^\star\|^2
\le
\|p_t-x^\star\|^2+2D\xi_t+\xi_t^2
=
\|p_t-x^\star\|^2+e_t.
\]
Projection is nonexpansive, so
\[
\|p_t-x^\star\|\le \|y_t-x^\star\|.
\]
Therefore
\[
\|x_{t+1}-x^\star\|^2
\le
\|x_t-x^\star\|^2
-2\eta_t\langle g_t,x_t-x^\star\rangle
+\eta_t^2\|g_t\|^2
+e_t.
\]
By the subgradient inequality,
\[
\Phi(x_t)-\Phi(x^\star)\le \langle g_t,x_t-x^\star\rangle.
\]
Hence
\[
2\eta_t(\Phi(x_t)-\Phi(x^\star))
\le
\|x_t-x^\star\|^2-\|x_{t+1}-x^\star\|^2+\eta_t^2\|g_t\|^2+e_t.
\]

Summing from \(t=1\) to \(T\) and using the standard AdaGrad telescoping argument yields
\[
\sum_{t=1}^T (\Phi(x_t)-\Phi(x^\star))
\le
\frac{D^2}{2\eta_T}
+\frac12\sum_{t=1}^T \eta_t\|g_t\|^2
+\frac12\sum_{t=1}^T \frac{e_t}{\eta_t}.
\]
Since \(\eta_T=D/\sqrt{S_T}\),
\[
\frac{D^2}{2\eta_T}=\frac{D}{2}\sqrt{S_T}.
\]
Also,
\[
\sum_{t=1}^T \eta_t\|g_t\|^2
=
D\sum_{t=1}^T \frac{\|g_t\|^2}{\sqrt{S_t}}
\le
2D\sqrt{S_T}.
\]
Finally, because \(\xi_t=\min\{D,\alpha\eta_t/(6D)\}\), we have
\[
e_t=2D\xi_t+\xi_t^2\le 3D\xi_t\le \frac{\alpha\eta_t}{2},
\]
so
\[
\frac12\sum_{t=1}^T \frac{e_t}{\eta_t}\le \frac{\alpha T}{4}.
\]
Combining the bounds,
\[
\sum_{t=1}^T (\Phi(x_t)-\Phi(x^\star))
\le
\frac{3D}{2}\sqrt{S_T}+\frac{\alpha T}{4}.
\]
By convexity,
\[
\Phi(\bar x_T)-\Phi(x^\star)
\le
\frac1T\sum_{t=1}^T (\Phi(x_t)-\Phi(x^\star))
\le
\frac{3D}{2T}\sqrt{S_T}+\frac{\alpha}{4}.
\]
Thus whenever \(3D\sqrt{S_T}\le \alpha T\), the returned point \(\bar x_T\) satisfies
\[
\Phi(\bar x_T)-\Phi(x^\star)<\alpha.
\]

To show finite termination, since \(\|g_t\|\le L\),
\[
S_T\le TL^2.
\]
Hence \(3D\sqrt{S_T}\le 3DL\sqrt T\), so the stopping condition is guaranteed once
\[
3DL\sqrt T\le \alpha T,
\]
i.e. once
\[
T\ge \left(\frac{3DL}{\alpha}\right)^2.
\]
This proves the claim.
\end{proof}

\begin{lemma}[Precise version of \Cref{lem:inexact-proj-W0}]
\label{lem:inexact-proj-W0-precise}
Let \(\WW\subseteq \mathbb{R}^d\) be a nonempty compact convex set, let
\(c\in \WW\), and let \(r>0\). Define
\[
U:=\WW\cap B(c,r)=\{w\in \WW:\|w-c\|_2\le r\}.
\]
Assume that Euclidean projection onto \(\WW\) can be computed exactly in time
\(\mathrm{T}_{\mathrm{proj}}(d)\). Then, for every \(y\in\mathbb{R}^d\) and
every \(\xi>0\), one can compute a point
\(\widetilde \Pi^\xi_U(y)\in U\) satisfying
\[
\bigl\|\widetilde \Pi^\xi_U(y)-\Pi_U(y)\bigr\|_2\le \xi
\]
using
\[
O\!\left(
\log\!\left(
1+\frac{\|y-c\|_2^2}{r\xi}
\right)
\right)
\]
calls to the projection oracle for \(\WW\), plus polynomial-time arithmetic in
\(d\). In particular, if projection onto \(\WW\) is polynomial-time computable,
then so is \(\xi\)-inexact projection onto \(U\).
\end{lemma}

\begin{proof}
Fix \(y\in\mathbb{R}^d\) and \(\xi>0\), and let
\[
x^\star:=\Pi_U(y).
\]
Since \(U\) is nonempty, compact, and convex, \(x^\star\) is well-defined and
unique.

If \(y=c\), then \(c\in U\) and \(x^\star=c\), so we can return \(c\). Hence
assume \(R:=\|y-c\|_2>0\).

\paragraph{Step 1: Multiplier representation.}
If \(\Pi_{\WW}(y)\in U\), then \(\Pi_U(y)=\Pi_{\WW}(y)\), so we are done.
Thus assume
\[
\|\Pi_{\WW}(y)-c\|_2>r.
\]

Consider the projection problem
\[
\min_{x\in \WW}\ \frac12\|x-y\|_2^2
\qquad\text{s.t.}\qquad
\frac12\bigl(\|x-c\|_2^2-r^2\bigr)\le 0.
\]
The feasible set is \(U\). Since \(c\in\WW\) and \(r>0\), the set
\(\WW\cap B(c,r)\) has nonempty relative interior in \(\operatorname{aff}(\WW)\):
indeed, because \(\operatorname{ri}(\WW)\) is dense in \(\WW\), there exists
\(u\in\operatorname{ri}(\WW)\) with \(\|u-c\|_2<r\). Therefore Slater's condition
holds relative to \(\operatorname{aff}(\WW)\), and the KKT conditions are
necessary and sufficient.

Hence there exists a multiplier \(\lambda^\star\ge 0\) such that
\[
0\in x^\star-y+N_{\WW}(x^\star)+\lambda^\star(x^\star-c),
\]
with complementary slackness
\[
\lambda^\star\bigl(\|x^\star-c\|_2^2-r^2\bigr)=0.
\]
Since \(\Pi_{\WW}(y)\notin U\), the ball constraint is active at \(x^\star\),
so \(\lambda^\star>0\) and \(\|x^\star-c\|_2=r\).

The stationarity condition is exactly the optimality condition for minimizing
over \(\WW\) the strongly convex function
\[
x\mapsto
\frac12\|x-y\|_2^2+\frac{\lambda^\star}{2}\|x-c\|_2^2.
\]
Completing the square gives
\[
x^\star
=
\Pi_{\WW}\!\left(
\frac{y+\lambda^\star c}{1+\lambda^\star}
\right).
\]

For \(\lambda\ge0\), define
\[
z_\lambda:=\frac{y+\lambda c}{1+\lambda},
\qquad
x_\lambda:=\Pi_{\WW}(z_\lambda),
\qquad
g(\lambda):=\|x_\lambda-c\|_2^2-r^2.
\]
The preceding argument shows that \(x^\star=x_{\lambda^\star}\) for some
\(\lambda^\star>0\) satisfying \(g(\lambda^\star)=0\).

\paragraph{Step 2: Monotonicity and bracketing.}
We claim that \(g\) is continuous and nonincreasing. Continuity follows from
continuity of \(z_\lambda\) and nonexpansiveness of projection onto \(\WW\).
For monotonicity, define
\[
\psi(\lambda)
:=
\min_{x\in\WW}
\left\{
\frac12\|x-y\|_2^2
+
\frac{\lambda}{2}\bigl(\|x-c\|_2^2-r^2\bigr)
\right\}.
\]
This is the pointwise infimum of affine functions of \(\lambda\), hence
\(\psi\) is concave. By Danskin's theorem,
\[
\psi'(\lambda)
=
\frac12\bigl(\|x_\lambda-c\|_2^2-r^2\bigr)
=
\frac12 g(\lambda)
\]
at every point of differentiability, and the same conclusion holds in the
subgradient sense. Since the derivative/subgradient of a concave function is
nonincreasing, \(g\) is nonincreasing.

By assumption,
\[
g(0)=\|\Pi_{\WW}(y)-c\|_2^2-r^2>0.
\]
Let
\[
\bar\lambda:=\frac{R}{r}.
\]
Because \(c\in\WW\), nonexpansiveness of projection gives
\[
\|x_{\bar\lambda}-c\|_2
=
\|\Pi_{\WW}(z_{\bar\lambda})-\Pi_{\WW}(c)\|_2
\le
\|z_{\bar\lambda}-c\|_2
=
\frac{R}{1+\bar\lambda}
=
\frac{Rr}{R+r}
<r.
\]
Thus \(g(\bar\lambda)<0\). Hence \(\lambda^\star\in(0,\bar\lambda)\).

\paragraph{Step 3: Bisection.}
Perform bisection on \([0,\bar\lambda]\), using the sign of \(g(\lambda)\).
After \(N\) iterations, we obtain
\[
0\le \lambda_-\le \lambda^\star\le \lambda_+\le \bar\lambda
\]
such that
\[
g(\lambda_-)\ge0,\qquad g(\lambda_+)\le0,\qquad
\lambda_+-\lambda_-\le \frac{\bar\lambda}{2^N}.
\]
Since \(g(\lambda_+)\le0\), the point \(x_{\lambda_+}\) belongs to \(U\). We
return
\[
\widetilde \Pi^\xi_U(y):=x_{\lambda_+}.
\]

\paragraph{Step 4: Error bound.}
Projection onto \(\WW\) is nonexpansive, so for any \(\lambda,\mu\ge0\),
\[
\|x_\lambda-x_\mu\|_2
\le
\|z_\lambda-z_\mu\|_2
=
\frac{|\lambda-\mu|}{(1+\lambda)(1+\mu)}\|y-c\|_2
\le
R|\lambda-\mu|.
\]
Therefore
\[
\|\widetilde \Pi^\xi_U(y)-x^\star\|_2
=
\|x_{\lambda_+}-x_{\lambda^\star}\|_2
\le
R(\lambda_+-\lambda_-)
\le
R\frac{\bar\lambda}{2^N}
=
\frac{R^2}{r\,2^N}.
\]
Thus it suffices to take
\[
N
=
\left\lceil
\log_2\!\left(
1+\frac{\|y-c\|_2^2}{r\xi}
\right)
\right\rceil.
\]
This proves the claimed oracle complexity and runtime.
\end{proof}

\section{Deferred Proofs for \Cref{sec:main-dop}}

For completeness, we provide formal pseudocode of our efficient ERM solver in \Cref{alg:recursive-localized-erm}.

\begin{algorithm}[t]
\caption{\textsc{Efficient-Recursive-Localized-OutputPert}$(Z,\eps,\lambda,w_0)$}
\label{alg:recursive-localized-erm}
\DontPrintSemicolon
\KwIn{Dataset \(Z=(z_1,\dots,z_m)\), privacy parameter \(\eps\), regularization parameter \(\lambda\), center \(w_0\in\WW\)}
\KwOut{A private point \(w_{\mathrm{DP}}\in\WW\)}

Run \Cref{alg:recursive-localized-outputpert-main} on \((Z,\eps,\lambda,w_0)\), implementing each call to \textsc{Localized-OutputPert} as follows: for the current localized set \(U_j\), minimize the localized joint convex objective \(\Phi_{Z,U_j}\) over \(U_j^{m+1}\) to the required deterministic accuracy using \Cref{alg:adaptive-inexact-proj-subgrad} with diameter bound \(\sqrt{m+1}R_j\) and the inexact projection oracle from \Cref{lem:inexact-proj-W0}\;

Let \(w_{\mathrm{loc},S}\) be the center returned by \Cref{alg:recursive-localized-outputpert-main}\;

\Return \(w_{\mathrm{DP}}:=w_{\mathrm{loc},S}\)\;
\end{algorithm}

\subsection{Proof of \Cref{prop:erm-main}}

\begin{proposition}[Re-statement of \Cref{prop:erm-main}]
Fix \(m \in \mathbb{N}\), \(w_0\in\WW\), \(\lambda>0\), and \(\rho\in(0,1/5)\).
Then, \Cref{alg:recursive-localized-erm} is \(\eps\)-differentially private. If
\(Z\sim P^m\), then its output \(w_{\mathrm{DP}}\) satisfies
\[
\Pr\!\left(
\|w_{\mathrm{DP}}-\hat w_\lambda(Z;w_0)\|
\le
c_{\mathrm{erm}}
\frac{1}{\lambda}
\left(
G_k\left(\frac{d}{m\eps}\right)^{1-\frac1k}
+
\frac{G_2}{\sqrt m}
\right)
\right)\ge 0.7.
\]
Moreover, with probability at least \(1-\rho\), the runtime is bounded by a polynomial in
\[
m,\ d,\ D,\ \lambda,\ \frac{1}{\lambda},\ \frac1\eps,\ G_k,\ \frac{1}{\rho}.
\]
Further, if \(\sup_{z\in\ZZ}\max_{w\in\WW}\|\nabla f(w,z)\|\)
is finite and polynomially bounded in the problem parameters, then the runtime is polynomial with probability \(1\).
\end{proposition}

\begin{proof}
Privacy and the distance guarantee follow directly from
\Cref{prop:bias-reduced-decomp}, since
\Cref{alg:recursive-localized-erm} outputs the final center
\(w_{\mathrm{loc},S}\) of \Cref{alg:recursive-localized-outputpert-main}. In
particular, with probability at least \(0.7\),
\[
\|w_{\mathrm{DP}}-\hat w_\lambda(Z;w_0)\|
\le
c_2\frac{G_k}{\lambda}
\left(\frac{d}{m\eps}\right)^{1-\frac1k}.
\]
This implies the displayed bound after increasing \(c_{\mathrm{erm}}\), since
the additional \(G_2/\sqrt m\) term is nonnegative.

It remains to justify the runtime. Each call to
\Cref{alg:localized-outputpert} requires computing an \(\alpha\)-approximate
minimizer of
\[
\hat F^{(w_0)}_{C_j,U_j,\lambda,Z}(w)
\]
over the current localized set \(U_j\). By
\Cref{lem:localized-joint-convex}, this is achieved by minimizing the jointly
convex objective \(\Phi_{Z,U_j}\) over \(U_j^{m+1}\). Since
\(U_j=\WW\cap\BB(w_{\mathrm{loc},j},r_j)\), \Cref{lem:inexact-proj-W0}
provides a polynomial-time inexact projection oracle for \(U_j\), and hence for
the product set \(U_j^{m+1}\).

Let
\[
A_Z:=\max_{1\le i\le m}\sup_{w\in\WW}\|\nabla f(w,z_i)\|.
\]
By the \(k\)-th moment assumption, \(A_Z<\infty\) with probability \(1\). For any
round \(j\), every subgradient of the localized joint objective
\[
\Phi_{Z,U_j}(w,y_1,\dots,y_m)
=
\frac1m\sum_{i=1}^m
\bigl[f(y_i,z_i)+C_j\|w-y_i\|\bigr]
+
\frac{\lambda}{2}\|w-w_0\|^2
\]
has norm at most an absolute constant times
\[
L_{Z,j}:=A_Z+C_j+\lambda D.
\]
Indeed, a subgradient has components
\[
g^{(w)}=\frac{C_j}{m}\sum_{i=1}^m s_i+\lambda(w-w_0),
\qquad
g^{(y_i)}=\frac1m(g_i-C_js_i),
\]
where \(s_i\in\partial\|w-y_i\|\), \(\|s_i\|\le1\), and
\(g_i\in\partial f(y_i,z_i)\), \(\|g_i\|\le A_Z\). Thus
\[
\|g^{(w)}\|\le C_j+\lambda D,
\qquad
\left(\sum_{i=1}^m\|g^{(y_i)}\|^2\right)^{1/2}
\le A_Z+C_j,
\]
up to absolute constants.

The product domain \(U_j^{m+1}\) has diameter at most
\[
\sqrt{m+1}\,\diam(U_j)\le \sqrt{m+1}\,D.
\]
Therefore, by \Cref{prop:adaptive-proj-subgrad-inexact}, each adaptive
projected-subgradient invocation terminates in finite time with probability
\(1\).

For the high-probability polynomial runtime bound, Markov's inequality and a
union bound give
\[
\Pr\!\left(
A_Z\le G_k\left(\frac{m}{\rho}\right)^{1/k}
\right)\ge 1-\rho.
\]
On this event, all realized Lipschitz parameters \(L_{Z,j}\) are polynomially
bounded in
\[
m,\ d,\ D,\ \lambda,\ \frac1\lambda,\ \frac1\eps,\ G_k,\ \frac1\rho,
\]
because the parameters \(S,R_j,C_j,\eps_j,\beta_j\) in
\Cref{alg:recursive-localized-outputpert-main} are explicit functions of these
quantities and \(S=O(\log\log(eD/R_\star))\). Hence the iteration complexity
from \Cref{prop:adaptive-proj-subgrad-inexact}, together with the projection
oracle from \Cref{lem:inexact-proj-W0}, gives the claimed polynomial runtime
bound with probability at least \(1-\rho\).

Finally, if
\[
L_*:=\sup_{z\in\ZZ}\max_{w\in\WW}\|\nabla f(w,z)\|
\]
is finite and polynomially bounded in the problem parameters, then \(A_Z\le
L_*\) deterministically. The same runtime bound therefore holds with probability
\(1\).
\end{proof}

\subsection{Proof of \Cref{thm:main-dop}}
\begin{theorem}[Re-statement of \Cref{thm:main-dop}]
Let \(\delta, \rho\in(0,1/5)\). 
\textbf{\textsc{Pop-Recursive-Localized-OutputPert}} is \(\eps\)-differentially private and with probability at least \(1-\delta\), its output \(\hat w\) satisfies
\[
F(\hat w)-F^*
\le
c_3 G_k D
\left(\frac{d\log(1/\delta)}{n\eps}\right)^{1-\frac1k}
+
c_4 G_2 D \sqrt{\frac{\log(1/\delta)}{n}},
\]
and its runtime is bounded with probability at least \(1-\rho\) by a polynomial in
$
n,\ d,\ D,\ \frac1\eps,\ G_k,\ \log\frac{n}{\delta},\ \frac{1}{\rho}.
$
Further, if \(\sup_{z\in\ZZ}\max_{w\in\WW}\|\nabla f(w,z)\|\)
is finite and polynomially bounded in the problem parameters, then its runtime is polynomial with probability \(1\).
\end{theorem}
\begin{proof}
Instantiate line~10 of \Cref{alg:pop-localize} with
\(\alg_{\mathrm{ERM}}=\textsc{Efficient-Recursive-Localized-OutputPert}\) from
\Cref{prop:erm-main}. By \Cref{prop:erm-main}, this phasewise primitive is \(\eps\)-DP and satisfies
\[
\Pr\!\left(
\|\alg_{\mathrm{ERM}}(Z,\eps,\lambda,w_0)-\hat w_\lambda(Z;w_0)\|
\le
c_{\mathrm{erm}}
\frac{1}{\lambda}
\left(
G_k\left(\frac{d}{m\eps}\right)^{1-\frac1k}
+
\frac{G_2}{\sqrt m}
\right)
\right)\ge 0.7.
\]

Applying \Cref{thm:localization-generic} yields an \(\eps\)-DP algorithm whose output \(\hat w\) satisfies
\[
F(\hat w)-F^*
\le
c_3 G_k D
\left(\frac{d\log(1/\delta)}{n\eps}\right)^{1-\frac1k}
+
c_4 G_2 D \sqrt{\frac{\log(1/\delta)}{n}}
\]
with probability at least \(1-\delta\).

For runtime, Algorithm~\ref{alg:pop-localize} performs
\[
T \lesssim \log(n)
\]
phases and
\[
J \lesssim \log(\log(n/\delta))
\]
repetitions per phase. By \Cref{prop:erm-main}, each regularized ERM call terminates in finite time with probability \(1\), so the full algorithm also terminates with probability \(1\).

Now fix any \(\rho\in(0,1/5)\). Allocate runtime failure probability
\[
\rho_{t,j}:=\frac{\rho}{2TJ}
\]
to each ERM call in phase \(t\), repetition \(j\). By \Cref{prop:erm-main}, with probability at least \(1-\rho_{t,j}\), the runtime of that call is bounded by a polynomial in
\[
m_t,\ d,\ D,\ \lambda_t,\ \frac1\eps,\ G_k,\ \frac{1}{\rho_{t,j}}.
\]
A union bound over all \(TJ\) calls shows that with probability at least \(1-\rho/2\), all phasewise ERM calls satisfy their polynomial runtime bounds simultaneously. Since \(m_t\le n\), \(T=O(\log n)\), \(J=O(\log(\log n/\delta))\), and the regularization schedule \(\lambda_t\) is explicit and geometric, the total runtime is bounded with probability at least \(1-\rho\) by a polynomial in
\[
n,\ d,\ D,\ \frac1\eps,\ G_k,\ \log\frac{n}{\delta},\ \frac{1}{\rho}.
\]
Here we use that, in the instantiation of \Cref{thm:localization-generic}, the base regularization parameter \(\lambda_1\) is chosen explicitly as a polynomial function of the problem parameters, so the dependence on \(1/\lambda_t\) is polynomially bounded. 

Finally, if $\sup_{z\in\ZZ}\max_{w\in\WW}\|\nabla f(w,z)\|$ is polynomially bounded, then every phasewise ERM call runs in deterministic polynomial time by \Cref{prop:erm-main}. Since there are only finitely many such calls, the entire algorithm runs in deterministic polynomial time.
\end{proof}

\subsection{Deterministic polynomial time under deterministic approximate localized-extension oracles}
\label{sec:wp1-subclasses}

The main algorithm of \Cref{thm:main-dop} is polynomial time with high
probability in the general heavy-tailed model, because the realized Lipschitz
constant of the joint reformulation is finite with probability \(1\) and
polynomially bounded with high probability. In this subsection we isolate a
structural condition under which the same statistical guarantee can be obtained
in deterministic polynomial time, even when the worst-case Lipschitz parameter
\[
\sup_{z\in\ZZ}\sup_{w\in\WW}\|\nabla f(w,z)\|
\]
is infinite.

The key point is that the joint reformulation is only needed because we do not
know how to access the Lipschitz extension directly. If, for the relevant
localized sets \(U=\WW\cap\BB(c,r)\), we can compute deterministic approximate
subgradients of the localized Lipschitz extension \(f_{C,U}\), then each
optimization subroutine in \Cref{alg:localized-outputpert} can be implemented
directly by projected subgradient descent on the localized Lipschitz-extension
objective. Since \(f_{C,U}\) is \(C\)-Lipschitz by construction, the resulting
iteration complexity is deterministically polynomial.

\begin{definition}[Deterministic approximate subgradient oracle for localized Lipschitz extensions]
\label{def:approx-localized-lipext-oracle}
Fix \(C>0\). We say that the loss class admits a deterministic polynomial-time
approximate subgradient oracle for localized \(C\)-Lipschitz extensions if the
following holds.

For every \(z\in\ZZ\), every set
\[
U=\WW\cap\BB(c,r)
\qquad(c\in\WW,\ r>0),
\]
every \(w\in U\), and every accuracy parameter \(B>0\), there is an algorithm
running in time polynomial in the problem parameters and \(1/B\) that returns a
vector
\[
\widetilde g_{C,U,B}(w,z)\in\R^d
\]
satisfying
\[
f_{C,U}(u,z)
\ge
f_{C,U}(w,z)
+
\langle \widetilde g_{C,U,B}(w,z),u-w\rangle
-
B
\qquad\forall u\in U,
\]
and
\[
\|\widetilde g_{C,U,B}(w,z)\|_2\le 2C.
\]
\end{definition}

\begin{lemma}[Projected subgradient method with biased subgradients and inexact projections]
\label{lem:proj-subgrad-biased}
Let \(K\subseteq\R^q\) be a nonempty compact convex set with
\(\diam(K)\le D_K\), and let \(\Phi:K\to\R\) be convex. Suppose that at each
query point \(x_t\in K\), an oracle returns \(\tilde g_t\in\R^q\) satisfying
\[
\Phi(u)\ge \Phi(x_t)+\langle \tilde g_t,u-x_t\rangle-B
\qquad\forall u\in K,
\]
and \(\|\tilde g_t\|_2\le L\). Consider iterates satisfying
\[
\|x_{t+1}-\Pi_K(x_t-\eta \tilde g_t)\|_2\le \xi,
\qquad t=1,\dots,T,
\]
and let \(\bar x_T:=T^{-1}\sum_{t=1}^T x_t\). Then
\[
\Phi(\bar x_T)-\min_{x\in K}\Phi(x)
\le
\frac{D_K^2}{2\eta T}
+
\frac{\eta L^2}{2}
+
B
+
\frac{D_K\xi}{\eta}
+
\frac{\xi^2}{2\eta}.
\]
In particular, if
\[
\eta=\frac{D_K}{L\sqrt T},
\qquad
T\ge \left(\frac{4D_KL}{\alpha}\right)^2,
\qquad
B\le\frac{\alpha}{4},
\qquad
\xi\le \min\left\{D_K,\frac{\alpha\eta}{6D_K}\right\},
\]
then \(\Phi(\bar x_T)-\min_K\Phi\le \alpha\).
\end{lemma}

\begin{proof}
Fix \(x^\star\in\argmin_{x\in K}\Phi(x)\), and let
\[
p_t:=\Pi_K(x_t-\eta\tilde g_t).
\]
Since \(x_{t+1},p_t,x^\star\in K\) and \(\diam(K)\le D_K\),
\[
\|x_{t+1}-x^\star\|^2
\le
\|p_t-x^\star\|^2+2D_K\xi+\xi^2.
\]
By nonexpansiveness of projection,
\[
\|p_t-x^\star\|^2
\le
\|x_t-\eta\tilde g_t-x^\star\|^2
=
\|x_t-x^\star\|^2
-2\eta\langle \tilde g_t,x_t-x^\star\rangle
+\eta^2\|\tilde g_t\|^2.
\]
Combining and rearranging gives
\[
\langle \tilde g_t,x_t-x^\star\rangle
\le
\frac{\|x_t-x^\star\|^2-\|x_{t+1}-x^\star\|^2}{2\eta}
+
\frac{\eta}{2}\|\tilde g_t\|^2
+
\frac{D_K\xi}{\eta}
+
\frac{\xi^2}{2\eta}.
\]
Using the biased subgradient inequality with \(u=x^\star\),
\[
\Phi(x_t)-\Phi(x^\star)
\le
\langle \tilde g_t,x_t-x^\star\rangle+B.
\]
Summing over \(t=1,\dots,T\), dividing by \(T\), using
\(\|x_1-x^\star\|\le D_K\), \(\|\tilde g_t\|\le L\), and convexity of \(\Phi\),
we obtain
\[
\Phi(\bar x_T)-\Phi(x^\star)
\le
\frac{D_K^2}{2\eta T}
+
\frac{\eta L^2}{2}
+
B
+
\frac{D_K\xi}{\eta}
+
\frac{\xi^2}{2\eta}.
\]
The final claim follows by substituting the stated choices.
\end{proof}

\begin{theorem}[Deterministic polynomial time from localized-extension oracles]
\label{thm:structured-subclasses}
Grant \Cref{ass:main}. Suppose, in addition, that the loss class admits a
deterministic polynomial-time approximate subgradient oracle for localized
\(C\)-Lipschitz extensions in the sense of
\Cref{def:approx-localized-lipext-oracle}. Then there exists a pure
\(\eps\)-DP algorithm such that, for every \(\delta\in(0,1/5)\), its output
\(\hat w\) satisfies
\[
F(\hat w)-F^*
\le
c_5 G_kD
\left(\frac{d\log(1/\delta)}{n\eps}\right)^{1-1/k}
+
c_6G_2D\sqrt{\frac{\log(1/\delta)}{n}}
\]
with probability at least \(1-\delta\). Moreover, the algorithm runs in
deterministic polynomial time.
\end{theorem}

\begin{proof}
We use the same outer population-localization framework as in
\Cref{thm:main-dop}. It suffices to implement the phasewise regularized ERM
primitive \Cref{alg:recursive-localized-erm} in deterministic polynomial time.

Fix one regularized ERM instance \((Z,\eps,\lambda,w_0)\) of sample size \(m\).
The recursive localization proof of \Cref{prop:bias-reduced-decomp} is
unchanged. The only issue is how to implement the approximate minimization step
inside each call to \Cref{alg:localized-outputpert}.

At round \(j\), the current set has the form
\[
U_j=\WW\cap\BB(c_j,r_j),
\]
and the algorithm needs a deterministic \(\alpha_j\)-approximate minimizer of
\[
\hat F^{(w_0)}_{C_j,U_j,\lambda,Z}(w)
=
\frac1m\sum_{i=1}^m f_{C_j,U_j}(w,z_i)
+
\frac{\lambda}{2}\|w-w_0\|^2
\]
over \(U_j\), where
\[
\alpha_j
=
\frac{C_j^2}{100\lambda m^2}
\min\left\{1,\frac{d^2\ell_j^2}{\eps_j^2}\right\}.
\]
By \Cref{def:approx-localized-lipext-oracle}, for each \(z_i\) and \(w\in U_j\)
we can compute a \(B\)-approximate subgradient
\(\widetilde g_i(w)\) of \(f_{C_j,U_j}(\cdot,z_i)\) over \(U_j\), with
\(\|\widetilde g_i(w)\|\le 2C_j\). Therefore
\[
\widetilde g(w)
:=
\frac1m\sum_{i=1}^m \widetilde g_i(w)+\lambda(w-w_0)
\]
is an \(B\)-approximate subgradient of
\(\hat F^{(w_0)}_{C_j,U_j,\lambda,Z}\), and
\[
\|\widetilde g(w)\|
\le
2C_j+\lambda D.
\]
Thus the objective is accessible through deterministically bounded biased
subgradients.

Projection onto \(U_j=\WW\cap\BB(c_j,r_j)\) is implemented to arbitrary accuracy
using \Cref{lem:inexact-proj-W0}. Applying
\Cref{lem:proj-subgrad-biased} with
\[
L=2C_j+\lambda D,
\qquad
B\le \alpha_j/4,
\]
and with projection accuracy chosen as in the lemma, returns a point
\(\tilde w_j\in U_j\) satisfying
\[
\hat F^{(w_0)}_{C_j,U_j,\lambda,Z}(\tilde w_j)
-
\min_{w\in U_j}\hat F^{(w_0)}_{C_j,U_j,\lambda,Z}(w)
\le
\alpha_j
\]
in deterministic polynomial time. The iteration complexity is polynomial
because \(C_j,\lambda,D,1/\alpha_j\), and the projection accuracies are explicit
polynomially bounded quantities in the phasewise parameters.

Thus every approximate minimization step in
\Cref{alg:recursive-localized-outputpert-main} can be implemented deterministically
in polynomial time. The privacy and distance guarantee of
\Cref{prop:bias-reduced-decomp} are unchanged, because the optimization
accuracy is deterministic. Hence the phasewise ERM primitive satisfies
\[
\Pr\!\left(
\|w_{\mathrm{DP}}-\hat w_\lambda(Z;w_0)\|
\le
c_{\mathrm{erm}}
\frac{1}{\lambda}
\left[
G_k\left(\frac{d}{m\eps}\right)^{1-1/k}
+
\frac{G_2}{\sqrt m}
\right]
\right)\ge 0.7
\]
and runs in deterministic polynomial time.

Plugging this primitive into \Cref{alg:pop-localize} and applying
\Cref{thm:localization-generic} gives the claimed excess-risk bound. Since the
outer algorithm makes only finitely many phasewise calls with polylogarithmic
overhead, the full runtime is deterministic polynomial time.
\end{proof}

\begin{proposition}[Polyhedral losses on compact explicitly SOCP-representable localized domains]
\label{prop:polyhedral-socp-oracle}
Assume that \(\WW\subseteq\R^d\) admits an explicit compact second-order-cone
representation with a strict interior point. Suppose moreover that for every
\(z\in\ZZ\),
\[
f(w,z)=\max_{j\in[M(z)]}\{a_j(z)^\top w+b_j(z)\}
\]
is polyhedral. Then the hypothesis of \Cref{thm:structured-subclasses} holds
for every localized set \(U=\WW\cap\BB(c,r)\) that admits a strict feasible
point in the corresponding SOCP representation. In particular, this holds for
the localized sets used by the algorithm after arbitrarily small harmless
radius enlargements.
\end{proposition}

\begin{proof}
Fix \(C>0\), \(z\in\ZZ\), a localized set \(U=\WW\cap\BB(c,r)\), a point
\(w\in U\), and \(B>0\). The localized Lipschitz extension is
\[
f_{C,U}(w,z)
=
\inf_{y\in U}
\left[
\max_{j\in[M(z)]}\{a_j(z)^\top y+b_j(z)\}
+
C\|w-y\|_2
\right].
\]
Introduce variables \(x,y,u,s,t\), where \(u\) denotes the auxiliary variable in
the SOCP representation of \(\WW\). Consider the conic program
\[
\begin{aligned}
\min_{x,y,u,s,t}\quad & s+Ct \\
\text{s.t.}\quad
& x+y=w, \\
& \|x\|_2\le t, \\
& s\ge a_j(z)^\top y+b_j(z)\qquad \forall j\in[M(z)], \\
& y\in\WW, \\
& \|y-c\|_2\le r.
\end{aligned}
\]
Using the assumed SOCP representation of \(\WW\), this is an explicit SOCP,
and its optimal value is exactly \(f_{C,U}(w,z)\).

By the strict-feasibility assumption, Slater's condition holds, so strong
duality holds. Let \(\nu\) be any dual-feasible point, and let \(g\) denote the
dual multiplier associated with the equality constraint \(x+y=w\). Since \(w\)
appears only in this equality constraint, the dual objective has the form
\[
D_w(\nu)=\langle g,w\rangle+\beta(\nu),
\]
where \(\beta(\nu)\) is independent of \(w\). Therefore, for every \(u'\in U\),
\[
D_{u'}(\nu)=D_w(\nu)+\langle g,u'-w\rangle.
\]

The SOC constraint \(\|x\|_2\le t\) has dual variable \((p,q)\) in the
self-dual second-order cone. Stationarity with respect to \(x\) and \(t\) gives
\[
g+p=0,
\qquad
q=C,
\]
and dual feasibility gives \(\|p\|\le q\). Hence
\[
\|g\|\le C.
\]

By deterministic interior-point methods for SOCP, for any \(B>0\) one can
compute in polynomial time in the problem parameters and \(1/B\) a dual-feasible
point \(\nu_B\) with
\[
D_w(\nu_B)\ge f_{C,U}(w,z)-B.
\]
Let \(g_B\) be its multiplier for \(x+y=w\). For every \(u'\in U\), weak duality
gives
\[
f_{C,U}(u',z)
\ge
D_{u'}(\nu_B)
=
D_w(\nu_B)+\langle g_B,u'-w\rangle
\ge
f_{C,U}(w,z)+\langle g_B,u'-w\rangle-B.
\]
Moreover, \(\|g_B\|\le C\). Thus \(g_B\) is a \(B\)-approximate subgradient of
\(f_{C,U}(\cdot,z)\) at \(w\) over \(U\), with norm at most \(C\). This proves
the required oracle condition.
\end{proof}

\begin{corollary}[Concrete practical examples]
\label{cor:polyhedral-practical}
Under the assumptions of \Cref{prop:polyhedral-socp-oracle}, the deterministic
polynomial-time guarantee of \Cref{thm:structured-subclasses} applies to:
\begin{enumerate}
\item affine losses \(f(w,z)=a(z)^\top w+b(z)\);
\item hinge / ReLU-type losses \(f(w,z)=\max\{0,a(z)^\top w+b(z)\}\);
\item absolute-value losses \(f(w,z)=|a(z)^\top w+b(z)|\).
\end{enumerate}
In particular, the resulting pure \(\eps\)-DP heavy-tailed SCO algorithm runs in
deterministic polynomial time for these losses on compact Euclidean balls,
ellipsoids, and polytopes, subject to the mild strict-feasibility condition for
the localized SOCPs.
\end{corollary}

\begin{proof}
Affine losses are polyhedral with one affine piece. Hinge/ReLU-type losses are
polyhedral with two pieces:
\[
\max\{0,a(z)^\top w+b(z)\}.
\]
Absolute-value losses are polyhedral because
\[
|a(z)^\top w+b(z)|
=
\max\{a(z)^\top w+b(z),-a(z)^\top w-b(z)\}.
\]
Euclidean balls, ellipsoids, and polytopes admit compact SOCP representations.
The claim follows from \Cref{prop:polyhedral-socp-oracle} and
\Cref{thm:structured-subclasses}.
\end{proof}

\section{High-probability lower bounds}
\label{sec:hp-lower-bounds}

We record here high-probability excess risk lower bounds for heavy-tailed DP SCO and mean estimation. These lower bounds are sharper by $\poly(\log(1/\delta))$ factors than the corresponding in-expectation lower bounds of \cite{bd14} and nearly match the upper bound in \Cref{thm:main-dop}.

The main goal of this section is to prove
\Cref{thm:sco-lb-mainbody}.

\paragraph{Risk notation.}
Let $\mathcal{P}$ be a class of distributions on $\mathbb{R}^d$.
Given an $\varepsilon$-DP mean estimator
$\widehat{\mu}:(\mathbb{R}^d)^n \to \mathbb{R}^d$ and $\zeta \in (0,1)$, define its
$(1-\zeta)$-\textit{quantile squared-error risk} over $\mathcal P$ by
\[
R^{\mathrm{mean}}_{n,\varepsilon,\zeta}(\widehat{\mu};\mathcal P)
:=
\sup_{P \in \mathcal{P}}
\inf\Bigl\{
r \ge 0 :
\Pr_{Z \sim P^n}\bigl(\|\widehat{\mu}(Z)-\mu_P\|_2^2 > r\bigr)\le \zeta
\Bigr\},
\]
where $Z=(z_1,\dots,z_n)$ and $\mu_P:=\mathbb E_P[z]$.

Let
\[
\mathcal{W} := B_2(0,D/2)=\{w\in\mathbb{R}^d:\|w\|_2\le D/2\},
\]
so that $\operatorname{diam}(\mathcal{W})=D$.
For a class $\mathcal P$ of distributions on $\mathbb R^d$, define the associated linear losses
\[
f(w,z):=\langle z,w\rangle,
\qquad
F_P(w):=\mathbb{E}_{z\sim P}[f(w,z)] = \langle \mu_P,w\rangle.
\]
Given an $\varepsilon$-DP algorithm $\mathcal A:(\mathbb R^d)^n\to \mathcal W$ for SCO, define its
$(1-\zeta)$-\textit{quantile excess risk} over $\mathcal P$ by
\[
R^{\mathrm{sco}}_{n,\varepsilon,\zeta}(\mathcal A;\mathcal P)
:=
\sup_{P \in \mathcal{P}}
\inf\Bigl\{
r \ge 0 :
\Pr_{Z \sim P^n}\bigl(F_P(\mathcal{A}(Z))-\inf_{w\in\mathcal W}F_P(w) > r\bigr)\le \zeta
\Bigr\}.
\]

\paragraph{Lower-bound strategy.}
The non-private term and the private term are witnessed by different hard distribution classes. This is sufficient for a minimax lower bound: since the risk is a supremum over
instances, a lower bound from one subclass and a lower bound from another imply
a lower bound by the maximum of the two obstructions, and hence by their sum up
to constants. Accordingly, we prove the two lower bounds separately and then combine them for any ambient class
$\mathcal P$ that contains both hard subclasses.
This separation is conceptually useful: the non-private term is a two-point phenomenon and is
most cleanly handled via the quantile version of Le Cam's method from \cite{ma2024high}, whereas
the private term is a packing phenomenon and is most cleanly handled by combining the pure-DP
packing lower bound of \cite{bd14} with a direct reduction from estimation to decoding.

\paragraph{Hard distribution classes.}
We will use two different distribution classes.

For the non-private term in our lower bound, define
\[
\mathcal P^{\mathrm{bdd}}(G_2)
:=
\left\{
P \text{ supported on }\{\pm G_2 e_1\}
\right\}.
\]
Every $P\in \mathcal P^{\mathrm{bdd}}(G_2)$ satisfies
\[
\mathbb E_P\|z\|_2^2 = G_2^2
\qquad\text{and}\qquad
\mathbb E_P\|z\|_2^k = G_2^k \le G_k^k
\]
whenever $G_2\le G_k$.

For the private term, let $V\subseteq \mathbb S^{d-1}$ be a finite packing of the unit sphere, and for parameters
$p\in(0,1]$ and $a>0$ define
\[
P_\nu := (1-p)\delta_0 + p\,\delta_{a\nu},
\qquad \nu\in V.
\]
We write
\[
\mathcal P^{\mathrm{pack}}_{k}(G_k;V,p)
:=
\{P_\nu:\nu\in V,\ a=G_k p^{-1/k}\}.
\]
Then every $P_\nu\in\mathcal P^{\mathrm{pack}}_{k}(G_k;V,p)$ satisfies
\[
\mathbb E_{P_\nu}\|z\|_2^k \le G_k^k.
\]

\paragraph{Reduction from SCO to mean estimation.}
We begin with the standard reduction from linear SCO lower bounds to mean-estimation lower bounds.

\begin{lemma}[Linear SCO reduces to mean estimation]
\label{lem:linear-reduction}
Let $\mathcal P$ be any family of distributions on $\mathbb{R}^d$ such that
$\|\mu_P\|_2=m$ for every $P\in\mathcal P$.
For any algorithm $\mathcal{A}:(\mathbb{R}^d)^n \to \mathcal{W}$, define
\[
\widehat{\mu}_{\mathcal A}(Z):=
\begin{cases}
- m\, \mathcal A(Z)/\|\mathcal A(Z)\|_2, & \mathcal A(Z)\neq 0,\\[1ex]
m e_1, & \mathcal A(Z)=0.
\end{cases}
\]
Then for every $P\in\mathcal P$,
\[
\|\widehat{\mu}_{\mathcal A}(Z)-\mu_P\|_2^2
\le
\frac{8m}{D}\Bigl(F_P(\mathcal A(Z))-\inf_{w\in\mathcal W}F_P(w)\Bigr)
\qquad\text{almost surely.}
\]
Consequently,
\[
R^{\mathrm{sco}}_{n,\varepsilon,\zeta}(\mathcal A;\mathcal P)
\ge
\frac{D}{8m}\,
R^{\mathrm{mean}}_{n,\varepsilon,\zeta}\!\bigl(\widehat{\mu}_{\mathcal A};\mathcal P\bigr).
\]
\end{lemma}

\begin{proof}
Fix $P\in\mathcal P$ and write $\mu_P = m u$ with $u\in\mathbb S^{d-1}$.
The minimizer of $F_P(w)=\langle \mu_P,w\rangle$ over
$\mathcal W=B_2(0,D/2)$ is $w^\star=-(D/2)u$, so
\[
\inf_{w\in\mathcal W}F_P(w)=-\frac{Dm}{2}.
\]
If $\mathcal A(Z)\neq 0$, then
\begin{align*}
F_P(\mathcal A(Z))-\inf_{w\in\mathcal W}F_P(w)
&=
\langle mu,\mathcal A(Z)\rangle+\frac{Dm}{2}\\
&\ge
\frac{Dm}{2}
\left(
1-\left\langle u,-\frac{\mathcal A(Z)}{\|\mathcal A(Z)\|_2}\right\rangle
\right)\\
&=
\frac{D}{4m}\|\widehat{\mu}_{\mathcal A}(Z)-\mu_P\|_2^2.
\end{align*}
If $\mathcal A(Z)=0$, then
\[
F_P(\mathcal A(Z))-\inf_{w\in\mathcal W}F_P(w)=\frac{Dm}{2},
\]
while $\|\widehat{\mu}_{\mathcal A}(Z)-\mu_P\|_2^2\le 4m^2$.
Combining the two cases gives the pointwise inequality. The final claim follows directly
from the definition of $R^{\mathrm{mean}}_{n,\varepsilon,\zeta}(\widehat{\mu}_{\mathcal A};\mathcal P)$.
\end{proof}

\paragraph{Mean estimation lower bound.}
\begin{proposition}[Non-private high-probability term]
\label{prop:nonprivate-mean}
There exists a universal constant $c>0$ such that for all
$k\ge 2$, $G_2>0$, $G_k\ge G_2$, $n\in\mathbb N$, and $\zeta\in(0,1/4]$,
every estimator $\widehat\mu:(\mathbb R^d)^n\to\mathbb R^d$ satisfies
\[
R^{\mathrm{mean}}_{n,\varepsilon,\zeta}(\widehat\mu;\mathcal P^{\mathrm{bdd}}(G_2))
\ge
c\,G_2^2\min\left\{\frac{\log(1/\zeta)}{n},\,1\right\}.
\]
\end{proposition}

\begin{proof}
For $\rho\in[0,1/2]$, define $P_+,P_-\in \mathcal P^{\mathrm{bdd}}(G_2)$ by
\[
P_+(z=G_2e_1)=\frac{1+\rho}{2},
\qquad
P_-(z=G_2e_1)=\frac{1-\rho}{2}.
\]
Then
\[
\mu_{P_+}=\rho G_2 e_1,
\qquad
\mu_{P_-}=-\rho G_2 e_1,
\]
so
\[
\|\mu_{P_+}-\mu_{P_-}\|_2=2\rho G_2.
\]
A direct calculation shows that
\[
\mathrm{KL}(P_+,P_-)
=
\frac{1+\rho}{2}\log\frac{1+\rho}{1-\rho}
+
\frac{1-\rho}{2}\log\frac{1-\rho}{1+\rho}
\le 4\rho^2
\]
for all $\rho\in[0,1/2]$. Hence
\[
\mathrm{KL}(P_+^{\otimes n},P_-^{\otimes n})\le 4n\rho^2.
\]
Choose
\[
\rho := c_0\min\left\{\sqrt{\frac{\log(1/\zeta)}{n}},\,1\right\}
\]
with $c_0>0$ small enough that
\[
\mathrm{KL}(P_+^{\otimes n},P_-^{\otimes n})
\le
\log\frac{1}{4\zeta(1-\zeta)}.
\]
Corollary~6 of \cite{ma2024high} then implies
\[
R^{\mathrm{mean}}_{n,\varepsilon,\zeta}(\widehat\mu;\{P_+,P_-\})
\ge c\,\rho^2 G_2^2.
\]
Since $\{P_+,P_-\}\subseteq \mathcal P^{\mathrm{bdd}}(G_2)$, this yields
\[
R^{\mathrm{mean}}_{n,\varepsilon,\zeta}(\widehat\mu;\mathcal P^{\mathrm{bdd}}(G_2))
\ge
c\,G_2^2\min\left\{\frac{\log(1/\zeta)}{n},\,1\right\}.
\]
\end{proof}

For the private term, we use a direct decoder reduction.

\begin{lemma}[Quantile estimation implies decoding on a packing]
\label{lem:quantile-to-decoder}
Let $\{\theta_\nu:\nu\in V\}\subset \mathbb R^d$ be such that
\[
\|\theta_\nu-\theta_{\nu'}\|_2 > 2r
\qquad\forall \nu\neq \nu'.
\]
Let $\{P_\nu:\nu\in V\}$ be distributions on $(\mathbb R^d)^n$, and let
$\widehat\theta:(\mathbb R^d)^n\to\mathbb R^d$ be any estimator.
Define the nearest-neighbor decoder
\[
\psi_{\widehat\theta}(Z)\in \arg\min_{\nu\in V}\|\widehat\theta(Z)-\theta_\nu\|_2.
\]
If for every $\nu\in V$,
\[
P_\nu^n\bigl(\|\widehat\theta(Z)-\theta_\nu\|_2\le r\bigr)\ge 1-\zeta,
\]
then
\[
\frac1{|V|}\sum_{\nu\in V} P_\nu^n\bigl(\psi_{\widehat\theta}(Z)\neq \nu\bigr)\le \zeta.
\]
\end{lemma}

\begin{proof}
Fix $\nu\in V$. On the event $\|\widehat\theta(Z)-\theta_\nu\|_2\le r$, we have
\[
\|\widehat\theta(Z)-\theta_{\nu'}\|_2
\ge
\|\theta_\nu-\theta_{\nu'}\|_2-\|\widehat\theta(Z)-\theta_\nu\|_2
>
r
\]
for every $\nu'\neq \nu$. Hence nearest-neighbor decoding returns $\nu$.
Therefore
\[
P_\nu^n(\psi_{\widehat\theta}(Z)\neq \nu)
\le
P_\nu^n(\|\widehat\theta(Z)-\theta_\nu\|_2>r)
\le
\zeta.
\]
Averaging over $\nu\in V$ gives the claim.
\end{proof}

\begin{proposition}[Pure-DP high-probability private term]
\label{prop:private-mean}
There exists a constant $c>0$ such that for all
$k\ge 2$, $\varepsilon\in(0,1]$, $\zeta\in(0,1/4]$, $G_k>0$, $n\in\mathbb N$, and $d\ge 1$,
every $\varepsilon$-DP estimator $\widehat\mu:(\mathbb R^d)^n\to\mathbb R^d$ satisfies
\[
R^{\mathrm{mean}}_{n,\varepsilon,\zeta}(\widehat\mu;\mathcal P^{\mathrm{pack}}_k(G_k))
\ge
c \,G_k^2
\min\left\{
\left(\frac{d+\log(1/\zeta)}{n\varepsilon}\right)^{2-2/k},
\,1
\right\},
\]
where
\[
\mathcal P^{\mathrm{pack}}_k(G_k)
:=
\bigcup_{V,p}\mathcal P^{\mathrm{pack}}_k(G_k;V,p)
\]
and the union is over all finite \(1/2\)-packings \(V\subseteq \mathbb S^{d-1}\)
with \(|V|\ge 2^{c d}\).
\end{proposition}

\begin{proof}
Fix a $1/2$-packing $V\subseteq S^{d-1}$ with $|V|\ge 2^{cd}$. Set
\[
p:=\min\left\{
\frac{1}{n\varepsilon}\log\frac{|V|-1}{4e\,\zeta},
\,1
\right\},
\qquad
a:=G_k p^{-1/k},
\]
and define
\[
P_\nu := (1-p)\delta_0 + p\,\delta_{a\nu},
\qquad \nu\in V.
\]
Then
\[
\mu_\nu:=\mu_{P_\nu} = pa\nu = G_k p^{1-1/k}\nu.
\]
Moreover,
\[
\mathbb E_{P_\nu}\|z\|_2^k
=
pa^k\|\nu\|_2^k
\le
G_k^k,
\]
so $P_\nu\in \mathcal P^{\mathrm{pack}}_k(G_k)$.

For $\nu\neq \nu'$, the packing separation implies
\[
\|\mu_\nu-\mu_{\nu'}\|_2
=
G_k p^{1-1/k}\|\nu-\nu'\|_2
\ge
\frac12 G_k p^{1-1/k}.
\]
Choose a sufficiently small constant $c_0$ such that
\[
r:=c_0 G_k p^{1-1/k}
\]
satisfies
\[
\|\mu_\nu-\mu_{\nu'}\|_2 > 2r
\qquad\forall \nu\neq \nu'.
\]

Suppose, for contradiction, that
\[
R^{\mathrm{mean}}_{n,\varepsilon,\zeta}(\widehat\mu;\{P_\nu:\nu\in V\})
<
r^2.
\]
Then, by definition of $R^{\mathrm{mean}}_{n,\varepsilon,\zeta}$,
\[
P_\nu^n\bigl(\|\widehat\mu(Z)-\mu_\nu\|_2\le r\bigr)\ge 1-\zeta
\qquad\forall \nu\in V.
\]
Hence, by Lemma~\ref{lem:quantile-to-decoder}, the nearest-neighbor decoder
$\psi_{\widehat\mu}$ associated with $\{\mu_\nu:\nu\in V\}$ satisfies
\[
\frac1{|V|}\sum_{\nu\in V} P_\nu^n\bigl(\psi_{\widehat\mu}(Z)\neq \nu\bigr)\le \zeta.
\]

On the other hand, \cite[proof of Proposition~4, via Theorem~3]{bd14} gives the following
pure-DP lower bound on average decoding error: for every $\varepsilon$-DP decoder $\psi$,
\[
\frac1{|V|}\sum_{\nu\in V} P_\nu^n\bigl(\psi(Z)\neq \nu\bigr)
\ge
\frac{q}{2(1+q)},
\qquad
q:=(|V|-1)e^{-\varepsilon\lceil np\rceil}.
\]
We now lower bound $q$.

If
\[
\frac{1}{n\varepsilon}\log\frac{|V|-1}{4e\,\zeta}\le 1,
\]
then by definition of $p$,
\[
(|V|-1)e^{-\varepsilon np}=4e\,\zeta.
\]
Since $\lceil np\rceil\le np+1$ and $\varepsilon\le 1$,
\[
q=(|V|-1)e^{-\varepsilon\lceil np\rceil}
\ge
e^{-1}(|V|-1)e^{-\varepsilon np}
=
4\zeta.
\]
Therefore
\[
\frac{q}{2(1+q)}
\ge
\frac{4\zeta}{2(1+4\zeta)}
\ge
\zeta
\qquad\text{because }\zeta\le \frac14.
\]
This contradicts the existence of a decoder with average error at most $\zeta$.

If instead
\[
\frac{1}{n\varepsilon}\log\frac{|V|-1}{4e\,\zeta}>1,
\]
then $p=1$, so
\[
P_\nu=\delta_{G_k\nu}.
\]
In this regime, the same decoder lower bound yields a constant lower bound on average decoding
error, which is at least $\zeta$ since $\zeta\le 1/4$. Again we obtain a contradiction.

Thus
\[
R^{\mathrm{mean}}_{n,\varepsilon,\zeta}(\widehat\mu;\{P_\nu:\nu\in V\})
\ge
c_0 \,G_k^2 p^{2-2/k}.
\]
Since $\{P_\nu:\nu\in V\}\subseteq \mathcal P^{\mathrm{pack}}_k(G_k)$, we get
\[
R^{\mathrm{mean}}_{n,\varepsilon,\zeta}(\widehat\mu;\mathcal P^{\mathrm{pack}}_k(G_k))
\ge
c_0 \,G_k^2 p^{2-2/k}.
\]
Finally, using $|V|\ge 2^d$,
\[
\log\frac{|V|-1}{4e\,\zeta}
\ge
c\,(d+\log(1/\zeta))
\]
for a universal constant $c>0$, which yields
\[
R^{\mathrm{mean}}_{n,\varepsilon,\zeta}(\widehat\mu;\mathcal P^{\mathrm{pack}}_k(G_k))
\ge
c\,G_k^2
\min\left\{
\left(\frac{d+\log(1/\zeta)}{n\varepsilon}\right)^{2-2/k},
\,1
\right\}.
\]
\end{proof}

\paragraph{SCO lower bound.} Finally, we translate the mean estimation lower bounds into an SCO lower bound.

\begin{theorem}[SCO lower bound -- Precise version of \Cref{thm:sco-lb-mainbody}]
\label{thm:sco-lb}
There exists a universal constant $c>0$ such that the following holds.
Let $\mathcal P$ be any class of distributions on $\mathbb R^d$ that contains both
$\mathcal P^{\mathrm{bdd}}(G_2)$ and $\mathcal P^{\mathrm{pack}}_k(G_k)$.
Then for all $k\ge 2$, $0 < G_2 \le G_k$, $\varepsilon\in(0,1]$, $\zeta\in(0,1/4]$, $n\in\mathbb N$, and $d\ge 1$,
every $\varepsilon$-DP algorithm $\mathcal A:(\mathbb R^d)^n\to\mathcal W$ satisfies:
there exists $P\in\mathcal P$ such that
\[
\Pr_{Z\sim P^n}\!\left(
F_P(\mathcal A(Z))-\inf_{w\in\mathcal W}F_P(w)
\ge
cD\,
\min\Biggl\{
G_2,\;
G_2\sqrt{\frac{\log(1/\zeta)}{n}}
+
G_k\left(\frac{d+\log(1/\zeta)}{n\varepsilon}\right)^{1-1/k}
\Biggr\}
\right)\ge \zeta.
\]
\end{theorem}

\begin{proof}
For the non-private term, apply Lemma~\ref{lem:linear-reduction} with
$\mathcal P=\{P_+,P_-\}\subseteq \mathcal P^{\mathrm{bdd}}(G_2)$ from the proof of
Proposition~\ref{prop:nonprivate-mean}. In that family,
\[
\|z\|_2 = G_2 \qquad\text{almost surely},
\]
hence \((\mathbb E\|z\|_2^j)^{1/j}=G_2\) for all \(j\ge2\).
Also,
\[
\|\mu_P\|_2 \asymp G_2\min\left\{\sqrt{\frac{\log(1/\zeta)}{n}},\,1\right\},
\]
and Proposition~\ref{prop:nonprivate-mean} gives a mean-estimation lower bound of the order of
its square. Lemma~\ref{lem:linear-reduction} therefore implies that for every $\varepsilon$-DP
algorithm $\mathcal A$ there exists $P_{\mathrm{np}}\in\mathcal P^{\mathrm{bdd}}(G_2)\subseteq\mathcal P$
such that
\[
\Pr_{Z\sim P_{\mathrm{np}}^n}\!\left(
F_{P_{\mathrm{np}}}(\mathcal A(Z))-\inf_{w\in\mathcal W}F_{P_{\mathrm{np}}}(w)
\ge
cD\,G_2\min\left\{\sqrt{\frac{\log(1/\zeta)}{n}},\,1\right\}
\right)\ge \zeta.
\]

For the private term, apply Lemma~\ref{lem:linear-reduction} with
$\mathcal P=\{P_\nu:\nu\in V\}\subseteq \mathcal P^{\mathrm{pack}}_k(G_k)$ from the proof of
Proposition~\ref{prop:private-mean}. In that family,
\[
\|\mu_{P_\nu}\|_2 = G_k p^{1-1/k},
\qquad
p\asymp \min\left\{\frac{d+\log(1/\zeta)}{n\varepsilon},\,1\right\},
\]
and Proposition~\ref{prop:private-mean} gives a mean-estimation lower bound of order
$G_k^2 p^{2-2/k}$. Lemma~\ref{lem:linear-reduction} therefore implies that for every
$\varepsilon$-DP algorithm $\mathcal A$ there exists
$P_{\mathrm{priv}}\in\mathcal P^{\mathrm{pack}}_k(G_k)\subseteq\mathcal P$ such that
\[
\Pr_{Z\sim P_{\mathrm{priv}}^n}\!\left(
F_{P_{\mathrm{priv}}}(\mathcal A(Z))-\inf_{w\in\mathcal W}F_{P_{\mathrm{priv}}}(w)
\ge
cD\,G_k\min\left\{
\left(\frac{d+\log(1/\zeta)}{n\varepsilon}\right)^{1-1/k},
\,1
\right\}
\right)\ge \zeta.
\]

Set
\[
A:=G_2\sqrt{\frac{\log(1/\zeta)}{n}},
\qquad
B:=G_k\left(\frac{d+\log(1/\zeta)}{n\varepsilon}\right)^{1-1/k}.
\]
The non-private construction gives a lower bound of order
\(D\min\{G_2,A\}\). The private construction gives a lower bound of order
\(D\min\{G_k,B\}\), and hence also of order \(D\min\{G_2,B\}\), since
\(G_2\le G_k\). Therefore, for every \(\varepsilon\)-DP algorithm
\(\mathcal A\), one of the two distributions \(P_{\mathrm{np}}\) or
\(P_{\mathrm{priv}}\) satisfies
\[
\Pr\!\left(
F_P(\mathcal A(Z))-\inf_{w\in\mathcal W}F_P(w)
\ge
cD\max\{\min(G_2,A),\min(G_2,B)\}
\right)\ge \zeta.
\]
Finally,
\[
\max\{\min(G_2,A),\min(G_2,B)\}
\ge
\frac12\min\{G_2,A+B\}.
\]
Reducing \(c\) proves the claimed bound.
\end{proof}

\begin{remark}[Tightness of \Cref{thm:sco-lb}]
Note that the trivial algorithm that outputs any fixed $w_0 \in \WW$ is $0$-DP and achieves excess risk $\le DG_2$ with probability $1$, by $G_2$-Lipschitz continuity of the population loss. Combining this observation with Theorem~\ref{thm:main-dop} shows that \Cref{thm:sco-lb} is nearly tight: the only part that is not tight is that $d \log(1/\delta)$ appears in the private optimization term of the upper bound whereas $d + \log(1/\delta)$ appears in the private optimization term of lower bound.    
\end{remark}

\end{document}

%% file: notation.tex
\usepackage{amsmath}
\usepackage{amsthm}
\usepackage{amssymb}
\usepackage{mathtools}
\usepackage{xcolor}
\usepackage{xspace}
\usepackage{bbm}
\usepackage{lmodern}
\usepackage[lined,boxed,ruled,norelsize,linesnumbered]{algorithm2e}
\usepackage{algorithmic}
\usepackage{aliascnt}
\usepackage{cleveref}
\usepackage{empheq}

\allowdisplaybreaks

\definecolor{b2}{RGB}{51,153,255}
\definecolor{mygreen}{RGB}{80,180,0}

\theoremstyle{plain}
\newtheorem{theorem}{Theorem}[section]

\newaliascnt{lemma}{theorem}
\newtheorem{lemma}[lemma]{Lemma}
\aliascntresetthe{lemma}

\newaliascnt{proposition}{theorem}
\newtheorem{proposition}[proposition]{Proposition}
\aliascntresetthe{proposition}

\newaliascnt{fact}{theorem}

\aliascntresetthe{fact}

\newaliascnt{claim}{theorem}

\aliascntresetthe{claim}

\newaliascnt{corollary}{theorem}
\newtheorem{corollary}[corollary]{Corollary}
\aliascntresetthe{corollary}

\newaliascnt{conjecture}{theorem}

\aliascntresetthe{conjecture}

\newaliascnt{assumption}{theorem}
\newtheorem{assumption}[assumption]{Assumption}
\aliascntresetthe{assumption}

\newaliascnt{condition}{theorem}

\aliascntresetthe{condition}

\newaliascnt{hypothesis}{theorem}

\aliascntresetthe{hypothesis}

\newaliascnt{question}{theorem}

\aliascntresetthe{question}

\newaliascnt{notation}{theorem}

\aliascntresetthe{notation}

\theoremstyle{definition}

\newaliascnt{definition}{theorem}
\newtheorem{definition}[definition]{Definition}
\aliascntresetthe{definition}

\newaliascnt{example}{theorem}

\aliascntresetthe{example}

\newaliascnt{problem}{theorem}

\aliascntresetthe{problem}

\newaliascnt{open}{theorem}

\aliascntresetthe{open}

\newaliascnt{observation}{theorem}

\aliascntresetthe{observation}

\theoremstyle{remark}

\newaliascnt{remark}{theorem}
\newtheorem{remark}[remark]{Remark}
\aliascntresetthe{remark}

\crefname{algocf}{algorithm}{algorithms}
\Crefname{algocf}{Algorithm}{Algorithms}

\crefname{theorem}{theorem}{theorems}
\Crefname{theorem}{Theorem}{Theorems}

\crefname{lemma}{lemma}{lemmas}
\Crefname{lemma}{Lemma}{Lemmas}

\crefname{proposition}{proposition}{propositions}
\Crefname{proposition}{Proposition}{Propositions}

\crefname{fact}{fact}{facts}
\Crefname{fact}{Fact}{Facts}

\crefname{claim}{claim}{claims}
\Crefname{claim}{Claim}{Claims}

\crefname{corollary}{corollary}{corollaries}
\Crefname{corollary}{Corollary}{Corollaries}

\crefname{conjecture}{conjecture}{conjectures}
\Crefname{conjecture}{Conjecture}{Conjectures}

\crefname{assumption}{assumption}{assumptions}
\Crefname{assumption}{Assumption}{Assumptions}

\crefname{condition}{condition}{conditions}
\Crefname{condition}{Condition}{Conditions}

\crefname{hypothesis}{hypothesis}{hypotheses}
\Crefname{hypothesis}{Hypothesis}{Hypotheses}

\crefname{question}{question}{questions}
\Crefname{question}{Question}{Questions}

\crefname{notation}{notation}{notations}
\Crefname{notation}{Notation}{Notations}

\crefname{definition}{definition}{definitions}
\Crefname{definition}{Definition}{Definitions}

\crefname{example}{example}{examples}
\Crefname{example}{Example}{Examples}

\crefname{problem}{problem}{problems}
\Crefname{problem}{Problem}{Problems}

\crefname{open}{open problem}{open problems}
\Crefname{open}{Open Problem}{Open Problems}

\crefname{observation}{observation}{observations}
\Crefname{observation}{Observation}{Observations}

\crefname{remark}{remark}{remarks}
\Crefname{remark}{Remark}{Remarks}

\crefname{equation}{}{}
\crefname{figure}{Fig.}{Figs.}
\crefname{table}{Table}{Tables}
\crefname{section}{Section}{Sections}
\crefname{subsection}{Section}{Sections}
\crefname{appendix}{Appendix}{Appendices}

\newcommand{\wh}{\widehat}

\newcommand{\ov}{\overline}

\newcommand{\eps}{\varepsilon}
\renewcommand{\phi}{\varphi}

\newcommand{\R}{\mathbb{R}}

\newcommand{\calO}{\mathcal{O}}

\renewcommand{\hat}{\wh}
\renewcommand{\bar}{\ov}

\newcommand{\BB}{\mathbb{B}}

\DeclareMathAlphabet{\mathpzc}{OT1}{pzc}{m}{it}

\newcommand{\expec}{\mathbb{E}}
\newcommand{\E}{\mathbb{E}}
\newcommand{\pr}{\mathbb{P}}

\newcommand{\ZZ}{\mathcal{Z}}
\newcommand{\WW}{\mathcal{W}}

\newcommand{\alg}{\mathcal{A}}

\DeclareMathOperator*{\argmin}{arg\,min}

\DeclarePairedDelimiterX{\xdivergence}[2]{(}{)}{
  #1\;\delimsize\|\;#2
}

\newcommand{\poly}{\mathrm{poly}}

\DeclareMathOperator{\diam}{diam}

%% file: arxiv_v2.bbl
\newcommand{\etalchar}[1]{$^{#1}$}
\begin{thebibliography}{LWMW25}

\bibitem[ADF{\etalchar{+}}21]{asi2021private}
Hilal Asi, John Duchi, Alireza Fallah, Omid Javidbakht, and Kunal Talwar.
\newblock Private adaptive gradient methods for convex optimization.
\newblock In {\em International Conference on Machine Learning}, pages 383--392. PMLR, 2021.

\bibitem[AFKT21]{AsiFeKoTa21}
Hilal Asi, Vitaly Feldman, Tomer Koren, and Kunal Talwar.
\newblock Private stochastic convex optimization: Optimal rates in {$\ell_1$} geometry.
\newblock In {\em ICML}, 2021.

\bibitem[ALD21]{asi2021adapting}
Hilal Asi, Daniel L{\'e}vy, and John~C Duchi.
\newblock Adapting to function difficulty and growth conditions in private optimization.
\newblock {\em Advances in Neural Information Processing Systems}, 34:19069--19081, 2021.

\bibitem[ALT24]{asi2024private}
Hilal Asi, Daogao Liu, and Kevin Tian.
\newblock Private stochastic convex optimization with heavy tails: Near-optimality from simple reductions.
\newblock {\em Advances in Neural Information Processing Systems}, 37:59174--59215, 2024.

\bibitem[BD14]{bd14}
Rina~Foygel Barber and John~C Duchi.
\newblock Privacy and statistical risk: Formalisms and minimax bounds.
\newblock {\em arXiv preprint arXiv:1412.4451}, 2014.

\bibitem[BFTT19]{bft19}
Raef Bassily, Vitaly Feldman, Kunal Talwar, and Abhradeep Thakurta.
\newblock Private stochastic convex optimization with optimal rates.
\newblock In {\em Advances in Neural Information Processing Systems}, volume~32, 2019.

\bibitem[BGM23]{bassily2023differentially}
Raef Bassily, Crist{\'o}bal Guzm{\'a}n, and Michael Menart.
\newblock Differentially private algorithms for the stochastic saddle point problem with optimal rates for the strong gap.
\newblock In {\em The Thirty Sixth Annual Conference on Learning Theory}, pages 2482--2508. PMLR, 2023.

\bibitem[BS16]{bun2016concentrated}
Mark Bun and Thomas Steinke.
\newblock Concentrated differential privacy: Simplifications, extensions, and lower bounds.
\newblock In {\em Theory of Cryptography Conference}, pages 635--658. Springer, 2016.

\bibitem[BST14]{bst14}
Raef Bassily, Adam Smith, and Abhradeep Thakurta.
\newblock Private empirical risk minimization: Efficient algorithms and tight error bounds.
\newblock In {\em 2014 IEEE 55th Annual Symposium on Foundations of Computer Science}, pages 464--473. IEEE, 2014.

\bibitem[CMS11]{chaud}
Kamalika Chaudhuri, Claire Monteleoni, and Anand~D Sarwate.
\newblock Differentially private empirical risk minimization.
\newblock {\em Journal of Machine Learning Research}, 12(3), 2011.

\bibitem[DHS11]{duchi2011adaptive}
John Duchi, Elad Hazan, and Yoram Singer.
\newblock Adaptive subgradient methods for online learning and stochastic optimization.
\newblock {\em Journal of machine learning research}, 12(7), 2011.

\bibitem[DMNS06]{dwork2006calibrating}
Cynthia Dwork, Frank McSherry, Kobbi Nissim, and Adam Smith.
\newblock Calibrating noise to sensitivity in private data analysis.
\newblock In {\em Theory of cryptography conference}, pages 265--284. Springer, 2006.

\bibitem[FKT20]{fkt20}
Vitaly Feldman, Tomer Koren, and Kunal Talwar.
\newblock Private stochastic convex optimization: optimal rates in linear time.
\newblock In {\em Proceedings of the 52nd Annual ACM SIGACT Symposium on Theory of Computing}, pages 439--449, 2020.

\bibitem[GTU23]{ganesh2023universality}
Arun Ganesh, Abhradeep Thakurta, and Jalaj Upadhyay.
\newblock Universality of langevin diffusion for private optimization, with applications to sampling from rashomon sets.
\newblock In {\em The Thirty Sixth Annual Conference on Learning Theory}, pages 1730--1773. PMLR, 2023.

\bibitem[HNXW22]{hu2022high}
Lijie Hu, Shuo Ni, Hanshen Xiao, and Di~Wang.
\newblock High dimensional differentially private stochastic optimization with heavy-tailed data.
\newblock In {\em Proceedings of the 41st ACM SIGMOD-SIGACT-SIGAI Symposium on Principles of Database Systems}, pages 227--236, 2022.

\bibitem[HUL13]{hiriart2013convex}
Jean-Baptiste Hiriart-Urruty and Claude Lemar{\'e}chal.
\newblock {\em Convex analysis and minimization algorithms I \& II}, volume 305.
\newblock Springer science \& business media, 2013.

\bibitem[KLZ22]{kamath2022improved}
Gautam Kamath, Xingtu Liu, and Huanyu Zhang.
\newblock Improved rates for differentially private stochastic convex optimization with heavy-tailed data.
\newblock In {\em International Conference on Machine Learning}, pages 10633--10660. PMLR, 2022.

\bibitem[LL25]{lowy2025dpbilevel}
Andrew Lowy and Daogao Liu.
\newblock Differentially private bilevel optimization: Efficient algorithms with near-optimal rates.
\newblock In {\em The Thirty-ninth Annual Conference on Neural Information Processing Systems}, 2025.

\bibitem[LLA24]{lowy2024faster}
Andrew Lowy, Daogao Liu, and Hilal Asi.
\newblock Faster algorithms for user-level private stochastic convex optimization.
\newblock {\em Advances in Neural Information Processing Systems}, 37:96071--96100, 2024.

\bibitem[LR21]{lowy2021output}
Andrew Lowy and Meisam Razaviyayn.
\newblock Output perturbation for differentially private convex optimization: Faster and more general.
\newblock {\em arXiv preprint arXiv:2102.04704}, 2021.

\bibitem[LR25]{lowy2025private}
Andrew Lowy and Meisam Razaviyayn.
\newblock Private stochastic optimization with large worst-case lipschitz parameter.
\newblock {\em Journal of Privacy and Confidentiality}, 15:1, 2025.

\bibitem[LWMW25]{lin2025purifying}
Yingyu Lin, Erchi Wang, Yi-An Ma, and Yu-Xiang Wang.
\newblock Purifying approximate differential privacy with randomized post-processing.
\newblock {\em arXiv preprint arXiv:2503.21071}, 2025.

\bibitem[MVS24]{ma2024high}
Tianyi Ma, Kabir~A Verchand, and Richard~J Samworth.
\newblock High-probability minimax lower bounds.
\newblock {\em arXiv preprint arXiv:2406.13447}, 2024.

\bibitem[SS{\etalchar{+}}12]{shalev2012online}
Shai Shalev-Shwartz et~al.
\newblock Online learning and online convex optimization.
\newblock {\em Foundations and Trends in Machine Learning}, 4(2):107--194, 2012.

\bibitem[TWZW21]{tao2021optimal}
Youming Tao, Yulian Wu, Peng Zhao, and Di~Wang.
\newblock Optimal rates of (locally) differentially private heavy-tailed multi-armed bandits.
\newblock {\em arXiv preprint arXiv:2106.02575}, 2021.

\bibitem[WX22]{wang2022differentially}
Di~Wang and Jinhui Xu.
\newblock Differentially private $\ell_1$-norm linear regression with heavy-tailed data.
\newblock {\em arXiv preprint arXiv:2201.03204}, 2022.

\bibitem[WXDX20]{wx20}
Di~Wang, Hanshen Xiao, Srinivas Devadas, and Jinhui Xu.
\newblock On differentially private stochastic convex optimization with heavy-tailed data.
\newblock In {\em International Conference on Machine Learning}, pages 10081--10091. PMLR, 2020.

\end{thebibliography}
